\pdfoutput=1

\documentclass[11pt]{article}

\usepackage[final]{EMNLP2022}

\usepackage{times}
\usepackage{latexsym}
\usepackage{amsmath}
\usepackage{booktabs}
\usepackage[T1]{fontenc}

\usepackage[utf8]{inputenc}
\usepackage{amsthm} 
\usepackage{breqn}
\usepackage{microtype}
\usepackage{graphicx}
\usepackage{subfigure}
\usepackage{inconsolata}

\newtheorem{theorem}{Theorem}[section]

\newtheorem{lemma}[theorem]{Lemma}

\title{vONTSS: vMF based semi-supervised neural topic 
modeling\\ with optimal transport}

\author{Weijie Xu,  Xiaoyu Jiang, Srinivasan H. Sengamedu, Francis Iannacci, Jinjin Zhao \\
  Amazon \\
  \texttt{weijiexu@amazon.com} \\}

\begin{document}
\maketitle
\begin{abstract}


Recently, Neural Topic Models (NTM), inspired by variational autoencoders, have attracted a lot of research interest; however, these methods have limited applications in the real world due to the challenge of incorporating human knowledge. This work presents a semi-supervised neural topic modeling method, vONTSS, which uses von Mises-Fisher (vMF) based variational autoencoders and optimal transport. 
When a few keywords per topic are provided, vONTSS in the semi-supervised setting generates potential topics and optimizes topic-keyword quality and topic classification. Experiments show that vONTSS outperforms existing semi-supervised topic modeling methods in classification accuracy and diversity. vONTSS also supports unsupervised topic modeling. Quantitative and qualitative experiments show that vONTSS in the unsupervised setting outperforms recent NTMs on multiple aspects: vONTSS discovers highly clustered and coherent topics on benchmark datasets. It is also much faster than the state-of-the-art weakly supervised text classification method while achieving similar classification performance.
We further prove the equivalence of optimal transport loss and cross-entropy loss at the global minimum.

\vspace{0.5em}
\hspace{.5em}\includegraphics[width=1.25em,height=1.25em]{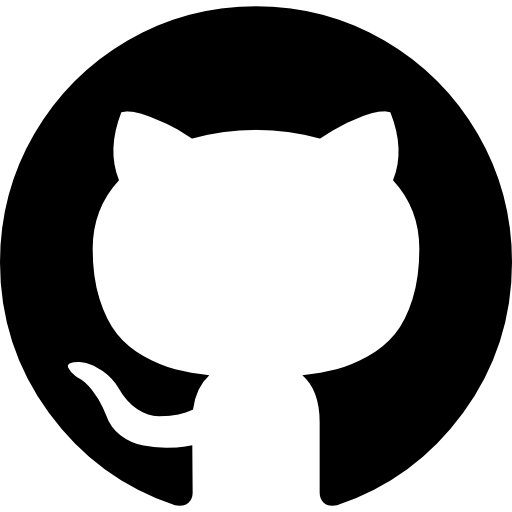}\hspace{.75em}\parbox{\dimexpr\linewidth-2\fboxsep-2\fboxrule}{\url{https://github.com/xuweijieshuai/vONTSS}}
\vspace{-.5em}
\end{abstract}
\section{Introduction}

Topic modeling methods such as \cite{blei2003latent} is an unsupervised approach for discovering latent structure in documents and achieving great performance \cite{blei2009nested}. Topic modeling methods take a list of documents as input. It generates the defined number of topics. It can further produce keywords and related documents for each topic.
In recent years, topic modeling methods have been widely used in many fields such as finance \cite{https://doi.org/10.1111/eufm.12326}, healthcare \cite{DBLP:journals/corr/abs-1711-10960}, education \cite{zhao2020targeted}, marketing \cite{Reisenbichler2019} and social science \cite{762586}. With the development of Variational Autoencoder (VAE)  \cite{https://doi.org/10.48550/arxiv.1312.6114}, Neural Topic Model \cite{miao2018discovering,dieng2020topic} has attracted attention as it enjoys better flexibility and scalability. However, recent research~\cite{hoyle2021automated} shows that the topics generated by these methods are not aligned with human perceptions.
To incorporate users' domain knowledge into the model, semi-supervised topic modeling methods become an active area of research \cite{mao2012sshlda,jagarlamudi-etal-2012-incorporating, gallagher2018anchored} and applications \cite{choi2017application, cao2019qos,kim2013mining}. Semi-supervised topic modeling methods take a few keywords as input and generate topics based on these keywords. Poeple use semi-supervised topic modeling methods because they want each topic include certain keywords and incorporate their domain expertise in their generated topics.
Traditional semi-supervised topic modeling methods fail to utilize semantic information of the corpus, causing low classification accuracy and high variance~\cite{chiu2022joint}. 



To solve these problems, we propose a von Mises-Fisher(vMF) based semi-supervised neural topic modeling method using optimal transport (vONTSS). 
We use the encoder-decoder framework for our model. The encoder uses modified vMF priors for latent distributions. The decoder uses a word-topic similarity matrix based on spherical embeddings. We use optimal transport to extend it to a semi-supervised version. vONTSS has the following enhancements:

1. We introduce the notion of temperature and make the spread of vMF distribution ($\kappa$) learnable, which leads to strong coherence and cluster-inducing properties.


2. vONT (In the rest of the paper, we use vONT to refer to the unsupervised topic model and vONTSS to semi-supervised version.) achieves the best coherence and clusterability compared to the state-of-the-art approaches on benchmark datasets.

3. We perform the human evaluation of the results for intrusion and rating tasks, and vONT outperforms other techniques.

4. Use of optimal transport to extend the stability of the model in the semi-supervised setting. The semi-supervised version is fast to train and achieves good alignment between keywords sets and topics. We also prove its theoretical properties.





5. In the semi-supervised scenario, we demonstrate the vONTSS achieves the best classification accuracy and lowest variance compared to other semi-supervised topic modeling methods.

6. We also show that vONTSS achieves similar performance as the state-of-the-art weakly text classification method while being much more efficient.



 
\section{Related Methods and Challenges}

\begin{table}[t]
\caption{Description of the notations used in this work}
\label{sample-table}

\vskip 0.15in

\begin{center}
\begin{small}
\scalebox{0.85} {
\begin{tabular}{lll}
\toprule
Notion & Description & Dimension\\
\hline
\midrule
$M$ & topic dimension  & \\
$V$ & vocabulary dimension  & \\
$Z$ & topic proportions  & $R^{M}$\\
$X$ & bow of a document  & $R^{V}$\\
$x$ & represent word &  $R$\\
$\theta$ & decoder  & \\
$\phi$ & encoder  & \\
$L_{\theta, \phi}(X)$ & loss function for NTM & \\
$e_{V}$ & word embedding matrix & $R^{V \times D}$\\
$e_{T}$ & topic embedding matrix & $R^{M \times D}$\\
$E$ &  topic-word matrix & $R^{M \times V}$\\
$\mu$ & vmf direction parameter   &$R^{M}$ \\
$\kappa$ &vmf concentration parameter  & $R$\\
$P$ & weight matrix for OT  & $R^{|T| \times |S|}$\\
$C$ & cost matrix for OT  & $R^{|T| \times |S|}$\\
$\eta$ & sample from vMF & $R^{M}$\\
$s$ & keywords set & \\
$S$ & group of keywords set & \\
$t$ & topic & \\
$T$ & group of topics & \\
$(s, t)$ & topic keywords set pair & \\
$L_{OT}$ & optimal transport loss & \\
$\lambda$ & entropy penalty weights & \\
$L_{ce}$ & cross-entropy loss & \\
$L_{X, T}$ & loss function for vONTSS & \\
$\alpha, \delta$ & parameters $L_{X, T}$ & \\
\hline
\end{tabular}}
\end{small}
\end{center}
\vskip -0.1in
\end{table}




\textbf{NTM} Variational Autoencoders (VAE)~\cite{https://doi.org/10.48550/arxiv.1312.6114} enable efficient variational inference. NTM~\cite{https://doi.org/10.48550/arxiv.1511.06038} uses $Z \in R^{M}$ as topic proportions over M topics and $X \in R^{V}$ to represent word count for the dataset with V unique words. NTM assumes that for any document, Z is generated from a document satisfying the prior distribution $p(Z)$ and X is generated by the conditional distribution $p_{\theta}(X|Z)$ where $\theta$ denotes a decoder. Ideally, we want to optimize the marginal likelihood $p_{\theta}(X) = \int{p(Z)p_{\theta}(X|Z)dZ}$. Due to the intractability of integration, NTM introduces $q_{\phi}(Z|X)$, a variational approximation to the posterior $p(Z|X)$. The loss function of NTM is: 
\begin{equation}
\begin{aligned}
L_{\theta, \phi} &= (-E_{q_{\phi}(Z|X)}[\log p_{\theta} (X|Z)] \\
    &+ KL[q_{\phi}(Z|X) || p(Z)] )
\end{aligned}
\label{eq10}
\end{equation} 
NTM usually utilizes a neural network with softmax to approximate $p_{\theta}(X|Z) := softmax(Wz)$ \cite{srivastava2017autoencoding}. NTM selects Gaussian \cite{miao2016neural}, Gamma~\cite{zhang2020whai} and Dirichlet distribution \cite{JMLR:v20:18-569} to approximate $p(Z)$. The second term Kullback-Leibler (KL) divergence regularizes  $q_{\phi}(Z|X)$ to be close to $p(Z)$. \textit{NTM has several problems in practice.} Firstly, it does not capture the semantic relationship between words. Secondly, the generated topics are not aligned with human interpretations. \cite{hoyle2021automated}.  Thirdly, using Gaussian prior may risk gravitating latent space toward the center and produce tangled representations among classes of documents. This is due to the fact that gaussian density presents a concentrated mass around the origin in low dimensional settings \cite{2013} and resembles a uniform distribution in high dimensional settings.

\textit{Extending NTM to semi-supervised version is also challenging.} $L_{\theta, \phi}$ is not always aligned with classification-related loss such as cross-entropy loss as identified by existing research \cite{https://doi.org/10.48550/arxiv.2204.03208}. To be specific, cross-entropy makes keywords sets align with assigned topics, while reconstruction loss($-E_{q_{\phi}(Z|X)}[\log p_{\theta} (X|Z)]$) makes latent space as representative as possible. Thus, existing semi-supervised NTM methods either are not stable \cite{wang2021neural,Harandizadeh_2022} or need certain adaptions~\cite{gemp2019weakly}. 

\textbf{Embedding Topic Model (ETM)} Pre-trained word embeddings such as Glove \cite{pennington-etal-2014-glove} and word2vec \cite{mikolov2013efficient} have the ability to capture semantic information, which is missing from basic bag-of-word (BoW) representations. 
They can serve as additional information to guide topic discovery. Dieng \cite{dieng2020topic} proposes ETM to use a vocabulary embedding matrix $e_{V} \in R^{V \times D}$ where D represents the dimension of word embeddings. The decoder $\phi$ learns a topic embedding matrix $e_{T} \in R^{M \times D}$. We denote topic to word distribution $ softmax(e_{T}  e_{V}^{T})$ as E  \begin{equation} p_{\theta}(X|Z) := Z \times E \label{eq11} \end{equation}  However since there exists some common words that are related to many other words, these common words' embeddings may be highly correlated with few topics' embeddings. Thus, \textit{ETM does not produce diverse topics}~\cite{zhao2020neural}. Besides, using pre-trained embeddings cannot help the model identify domain-specific topics. For example, topics related to COVID-19 are more likely to be expressed by a few topics instead of one single topic using pre-trained Glove embeddings~\cite{pennington2014glove} since COVID-19 is not in the embeddings.

\textbf{von Mises-Fisher} In low dimensions, the Gaussian density presents a concentrated probability mass around the origin. This is problematic when the data is partitioned into multiple clusters. An ideal prior should be non-informative and uniform over the parameter space. Thus, the von Mises-Fisher(vMF) is used in VAE. vMF is a distribution on the (M-1)-dimensional sphere in $R^{M}$, parameterized by $\mu \in R^{M}$ where $||\mu|| = 1$ and a concentration parameter $\kappa \in R_{\geq 0}$. The probability density function of the vMF distribution for $z \in R^{D}$ is defined as:

$$q(Z|\mu, \kappa) = C_{M}(\kappa) exp(\kappa\mu^{T}Z)$$
$$C_{M}(\kappa) = \frac{\kappa^{\frac{M}{2} - 1}}{(2\pi)^{\frac{M}{2}} I_{\frac{M}{2} - 1}(\kappa)} + \log 2$$

where $I_{v}$ denotes the modified Bessel function of the first kind at order v. The KL divergence with vMF(., 0) \cite{davidson2018hyperspherical} is

$$KL(vMF(\mu, \kappa)|vMF(.,0)) = \kappa\frac{I_{\frac{M}{2}}(\kappa)}{I_{\frac{M}{2}-1}(\kappa)} $$
$$+ (\frac{M}{2} - 1) \log \kappa - \frac{M}{2} \log (2\pi)  - \log I_{\frac{M}{2}-1}(\kappa) $$ $$+ \frac{M}{2} \log \pi + \log 2 + \log \Gamma(\frac{M}{2})$$

vMF based VAE has better clusterability of data points especially in low dimensions \cite{guu2018generating}. However, vMF distribution has limited expressibility when its sample is translated into a probability vector. Due to the unit constraint, $softmax$ of any sample of vMF will not result in high probability on any topic even under strong direction $\mu$. For example, when topic dimension $M$ equals to 10, the highest topic proportion of a certain topic is 0.23. \textit{Most of vMF-based topic modeling methods are not VAE based and very slow to train as summarized in Appendix~\ref{vMF}.}

\section{Proposed Methods}


The architecture of vONTSS is shown in Figure~\ref{fig:people4}. At a high level, our encoder network $\phi$ transforms the BoW representation of the document $X_{d}$ into a latent vector generated by vmf distribution and generates a sample $\eta_{d}$. We then apply a temperature function $\tau$ and softmax on this sample to get a probabilistic topic distribution $z_{d}$. Lastly, our decoder uses a modified topic-word matrix $E$ to reconstruct $X_{d}$'s BoW representation. To extend into semi-supervised setting, we leverage optimal transport to match keywords' set with topics. The encoder network $\phi$ and generative model parameter $\theta$ are learned jointly during the training process.


\begin{figure}
\hspace*{-2.5em}\includegraphics[scale=0.21]{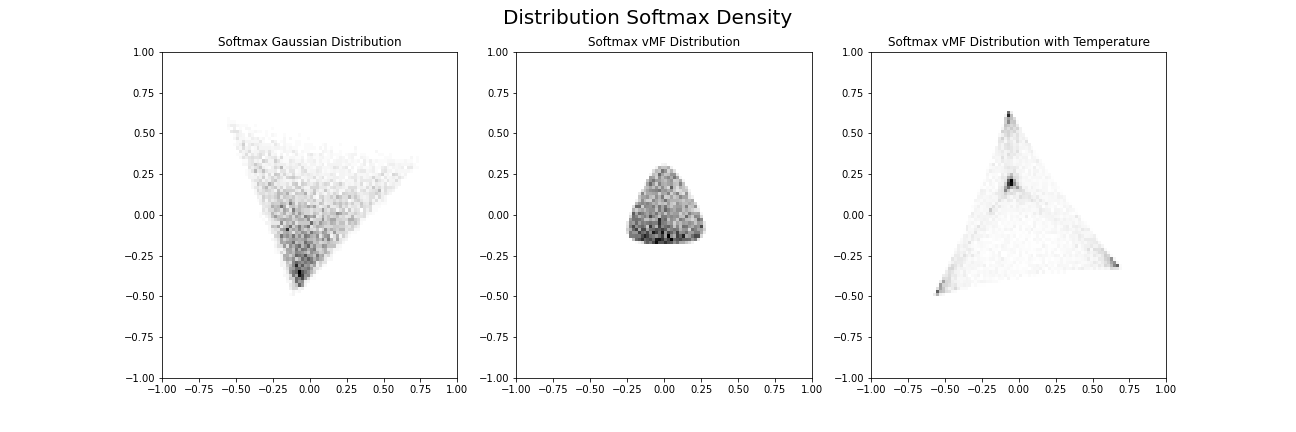}

\caption{2-D PCA projection of empirical CDF $softmax(\eta)$ where from left to right $\eta \sim \mathcal{N}(0, I)$, $\eta \sim vMF(.,0)$ and  $\eta \sim 10 * vMF(.,0)$) respectively. From the heatmap, we observe a white hole in the middle, which denotes the unreachable probability vector from each distribution. Gaussian is mean-centered, while basic vMF tends to cluster around a small rounded triangular area due to its unity constraint. vMF with radius equals to 10 is even more expressive than Gaussian while still retaining edge weights, inducing separability among different topics.} 
\label{fig:1}
\end{figure}

To overcome entangled topic latent space introduced by Gaussian distribution and limited expressibility of vMF distribution, we make two improvements: 1. Introduce a temperature function $\tau(\eta_{i})$ prior to $softmax()$ to modify the radius of vMF distribution. 2. Set $\kappa$ to a learnable parameter to flexibly infer the confidence of particular topics during training. 

\textbf{Encoder Network Temperature Function}
To alleviate concerns regarding expressibility while inducing separability among topics, we modify the radius of vMF distribution. We use a temperature function to represent the radius.  As shown in Figure~\ref{fig:1}, unmodified vMF distribution has limited expressiveness. For instance, Gaussian posteriors can express a topic probability vector of [0.98, 0.01, 0.005, 0.0003, 0.0002], while vMF can't due to the unity constraint. In practice, if we change the radius to 10, the network can learn more polarized topics distribution as shown in the right plot in Figure~\ref{fig:1}. The influence of different radii is analyzed in Appendix~\ref{radius}. 
Given equally powerful learning networks of distributions' parameters, vMF with different radii learns richer and more nuanced structures in their latent representations than a Gaussian counterpart (Appendix \ref{appendix:appendix2}). 

\begin{figure}
\hspace*{-2.6 em}\includegraphics[scale = 0.22]{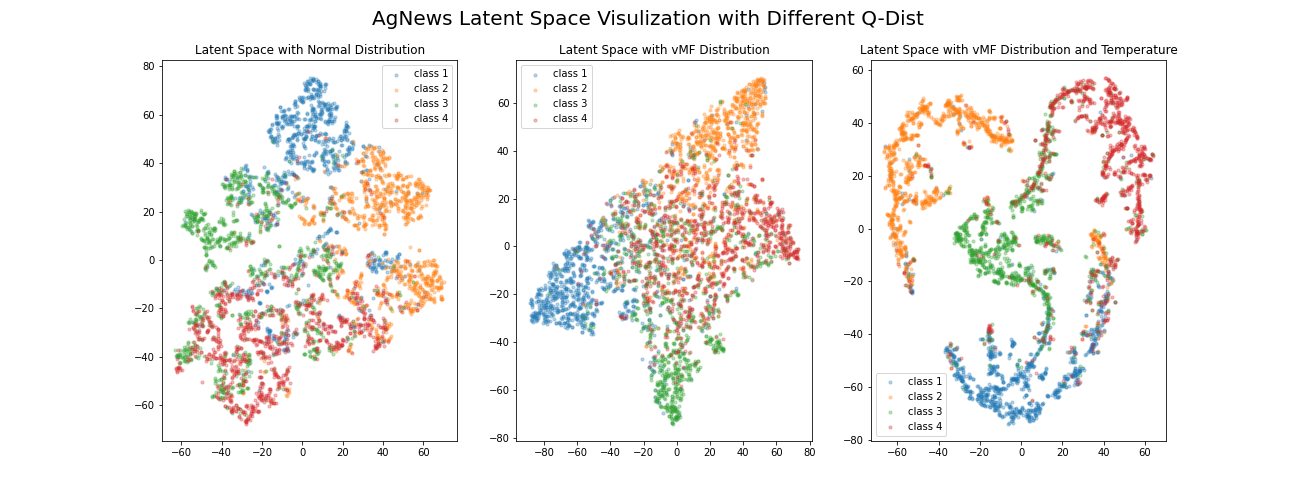}

\caption{2-D TSNE projection of randomly sampled $\eta$ from latent spaces under different posterior-distributions. From left to right are Gaussian, vMF with fixed k, vONT Each color represents a different topic. All encoders are trained on AgNews dataset with the same network structure. }
\label{fig:2}
\end{figure}

\textbf{Learnable $\kappa$} To further improve the clusterability, we convert $\kappa$ from a fixed value to a learnable parameter. The KL divergence of vMF distribution makes the distribution more concentrated while not influencing the direction of latent distribution. This makes the result more clustered. For Gaussian distribution, KL divergence penalizes the polarization of latent distribution (Appendix \ref{theory}). This makes the Gaussian distribution less clustered. 
To illustrate this, we randomly sampled encoded documents' latent distributions from AgNews Dataset~\cite{zhang2016characterlevel} after training with both latent distributions, as shown in Figure~\ref{fig:2}.
For the Gaussian distribution, we see that documents belonging to different topics are entangled around the center, causing the inseparability of topics during both the training and inference stage. vMF distribution, on the hand, repels four document classes into different quadrants, presents more structures when compared to Gaussian distribution, and creates better separable clusters. Detailed ablation study can be found in Appendix~\ref{kappas}

\textbf{Decoder Network}
Our decoder follows ETM's construction and uses the embedding $e_{V}$ and $e_{T}$ to generate a topic-word matrix $E$. 
One distinction between our decoder and ETM's decoder is that we generate the word embeddings by training a spherical embedding on the dataset. Spherical embeddings perform well in word similarity evaluation and document clustering \cite{meng2019spherical}, which further improves the clusterability of the topic modeling methods. 

We also keep word embeddings fixed during the topic modeling training process for two reasons. Firstly, 
keeping word embeddings fixed can alleviate sparsity issues \cite{pmlr-v80-zhao18a}. Additionally, vMF based VAE tends to be less expressive in high dimensions due to limited variance freedom \cite{davidson2018hyperspherical}. Keeping the embedding fixed can make topics more separable in higher-dimension settings and improve topic diversity.

\begin{figure}
\hspace*{-0.4em}\includegraphics[scale = 0.224]{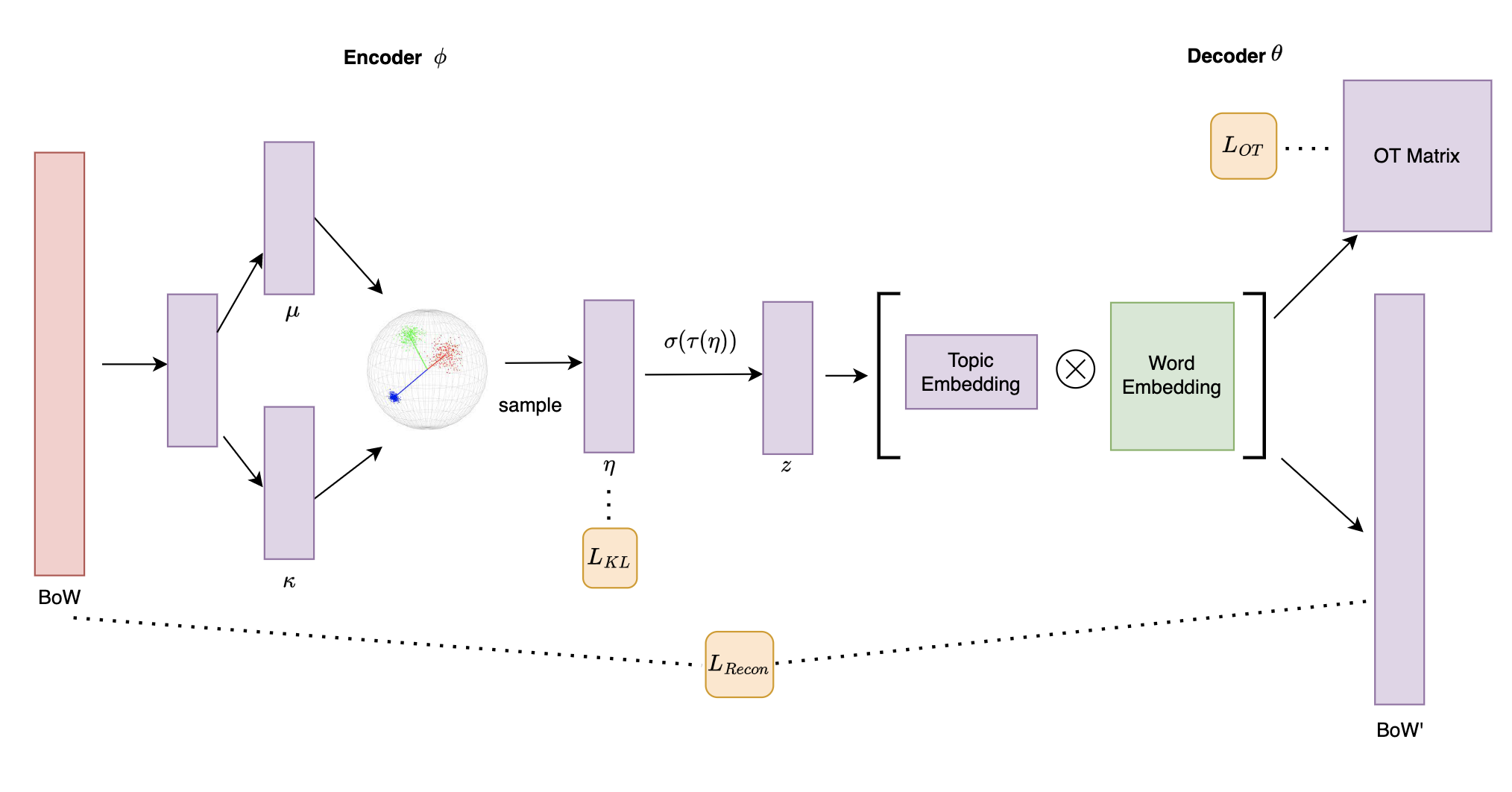}
\caption{Architecture of the model. Purple represents the part of the network that can be trained. $L_{recon}$ represents reconstruction loss. $L_{KL}$ represents KL divergence. $L_{OT}$ represents optimal transport loss}
\label{fig:people4}
\end{figure}

\textbf{Loss Function for vONTSS}
In semi-supervised settings, the user specifies sets of keywords $S$ associated with topics $T$. Let $(s, t)$ represent a keyword set and a topic pair, where each keyword $x \in s$ is labeled by topic $t$. Instead of training a separate neural network for the semi-supervised extension of NTM, we use the topic-word matrix (decoder $\theta$) to represent the probability of a word x given topic t. 

M1 + M2 is a semi-supervised model used in VAE. We adapt the M1 + M2 model framework \cite{kingma2014semi}. 
Under the assumption that $p_{\theta}(x, t, z) = p_{\theta}(x|z)  p_{\theta}(t|x)  p(z)$, our loss function can be approximated as \begin{equation}L(X, T) = L_{\theta, \phi}(X) - \alpha H[q_{\phi}(X|T)] + \delta L_{ce} \label{eq2} \end{equation} 
\begin{equation}L_{ce} = - \sum_{(s, t) \in (S, T)}E_{x \in s} \log q_{\theta}(x|t) \label{eq3} \end{equation}

For topic i and word j, we let $q_{\theta}(x_{j}|t_{i}) = E_{i, j}$ where $E$ is the topic-word matrix. $ H[q_{\phi}(X|T)] $ is entropy of $q_{\theta}(X|T)$. We can consider it as a regularization term. 

Optimizing the current model is hard because we have 3 objectives to minimize(cross-entropy, KL Divergence, and reconstruction loss) and they are not aligned with each other. To validate our point, we find out that if we make radius parameters learnable, the classification metric performs worse even if it decreases the reconstruction loss(Appendix~\ref{appendix:temperature}). If we apply cross-entropy at the beginning, topic embeddings get stuck into the center of selected keywords' embeddings, which makes the model overfitting. If we first train an unsupervised vONT, we need to find a way to match keywords and trained topics. If we match them based on their cosine similarity, different keywords may match to the same topics. This makes performance unstable. 
To deal with these challenges, we decide to use a two-stage training process and do not specify labeled keywords to topics at the beginning. vONTSS first optimizes $L_{\theta, \phi}(X) - \alpha H[q_{\phi}(X|T)]$ till convergence, then jointly optimizes $L(X, T)$ for few epochs. This makes our method easier to optimize, less time-consuming, and suitable for interactive topic modeling \cite{hu2014interactive}. To optimize $L_{ce}$ after stage 1, we need to pair topics and keyword sets. Existing methods such as Gumbel softmax prior \cite{jang2016categorical} often lead to instability, while naive matching by $q_{\phi}(x|t)$ may give us redundant topics.

\textbf{Optimal Transport for vONTSS} Optimal Transport (OT) distances \cite{chen2019improving, torres2021survey} have been widely used for comparing the distribution of probabilities. Specifically, let $U(r,c)$ be the set of positive $m \times n$ matrices for which the rows sum to r and the sum of the column to c:
$U(r, c) = \{P \in R_{>0}^{m \times n}|P 1_{t} = r, P^{T} 1_{s} =c \}$. For each position t, s in the matrix, it comes with a cost $C_{t,s}$. Our goal is to solve $d_{C}(r, c) = \min_{P \in U(r, c)} \sum_{t, s} P_{t,s} C_{t,s}$. To make distribution homogeneous \cite{cuturi2013sinkhorn}, we let \begin{equation}d_{C}^{\lambda}(r, c) = \min_{P \in U(r, c)} \sum_{t,s} P_{t,s} C_{t, s} - \frac{1}{\lambda} h(P)\label{eq12} \end{equation}  
\begin{equation}h(P) = - \sum_{t,s} P_{t,s} \log P_{t,s}\label{eq13} \end{equation} OT has achieved good robustness and semantic invariance in NLP related tasks \cite{chen2019improving}. Optimal transport has been used in topic modeling to replace KL divergence~\cite{zhao2020neural,Huynh2020OTLDAAG,Wang2022RepresentingMO}  or create topic embeddings \cite{xu2018distilled} as discussed in Appendix \ref{vMF}. It has not been used for extending topic modeling to semi-supervised cases. \\
To better match topic and keywords set, we approximate $L_{ce}$ using optimal transport. We choose sinkhorn distance since it has an entropy term, which makes our trained topics more coherent and stable. Our goal is to design the loss function that is aligned with derived cross-entropy loss at the global minimum.  To be specific, the raw dimension of our cost matrix is equal to the dimension of topics and the column dimension of the cost matrix equals to the dimension of keywords group. We denote each entry in the M matrix in optimal transport as, \begin{equation}C_{t, s} = - E_{x \in s}\log(q_{\theta}(x|t))\label{eq15}\end{equation} where t is the topic and x is the word in a keywords group s. The model uses sinkhorn distance and restricts the sum of each column and row of $P$ to 1. We give the model an entropy penalty term to make sure each topic is only related to one group of keywords. Thus, \begin{equation} L_{OT} =   \min_{P \in U(|T|, |S|)} \sum_{t,s} P_{t,s} C_{t, s} - \frac{1}{\lambda} h(P) \label{eq4} \end{equation} where $\lambda$ controls the entropy penalty. The first term is similar to $L_{ce}$ approximation, and the second term makes the result homogeneous. To lower the second term, each keyword should be highly correlated to one topic while not/negatively correlated with others. This further separates the topics and improves the topic diversity. We further show that $L_{OT} = L_{ce}$ when $L(X, T)$ is minimized. 
\begin{lemma} When L(X, T) reaches the global minimal. For any $(s, t)$, $(s', t') \in (S, T)$: 
\begin{equation}
\begin{aligned}
    &E_{x \in s} \log q_{\phi}(x|t) +  E_{x \in s'} \log q_{\phi}(x|t')  \\
    &- (E_{x \in s'} \log q_{\phi}(x|t)) + E_{x \in s} \log q_{\phi}(x|t')) >= 0
\end{aligned}
\label{eq20}
\end{equation}

\end{lemma}

\begin{theorem}

When L(X, T) reaches the global minimal, $$L_{OT} = L_{ce}  $$
\end{theorem}
Appendix \ref{appendix:proof} contains the proof.

\begin{table*}[ ]
\caption{\label{tab:people20} Clusterability metrics for vONT. The number of topics is 20. The best and second-best scores of each dataset are highlighted in boldface and with an underline, respectively. Figure~\ref{fig:people20} in the appendix shows the variation of the metrics as a function of a number of topics. It is hard to get Km-Purity for ProdLDA. Since it does not perform well for Top-purity, we do not think it will perform well on Km-purity. Thus, we ignore that result.}

\scalebox{0.49} {
\centering
\begin{tabular}{c|cccc|cccc|cccc|}
\hline
\multicolumn{1}{r}{}\vline &
\multicolumn{4}{c}{\textbf{AgNews}}\vline  & \multicolumn{4}{c}{\textbf{20News}} \vline & \multicolumn{4}{c}{\textbf{DBLP}}\vline\\                                                                                                                                                                                                                       

\multicolumn{1}{l}{\textbf{Method}} \vline & \multicolumn{1}{l}{\textbf{Km-Purity}} & \multicolumn{1}{l}{\textbf{Km-NMI}} & \multicolumn{1}{l}{\textbf{Top-Purity}} & \multicolumn{1}{l}{\textbf{Top-NMI}} \vline& \multicolumn{1}{l}{\textbf{Km-Purity}} & \multicolumn{1}{l}{\textbf{Km-NMI}} & \multicolumn{1}{l}{\textbf{Top-Purity}} & \multicolumn{1}{l}{\textbf{Top-NMI}} \vline& \multicolumn{1}{l}{\textbf{Km-Purity}} & \multicolumn{1}{l}{\textbf{Km-NMI}} & \multicolumn{1}{l}{\textbf{Top-Purity}} & \multicolumn{1}{l}{\textbf{Top-NMI}} \vline\\
\hline
\textbf{ETM}                                               & 0.570$\pm$0.023                                                    & 0.160$\pm$0.021                                                 & 0.556$\pm$0.036                                                     & 0.126$\pm$0.024                                                  & 0.689$\pm$0.028                                                    & 0.332$\pm$0.027                                                 & 0.731$\pm$0.037                                                     & 0.369$\pm$0.051                                                  & 0.217$\pm$0.023                                                    & 0.268$\pm$0.022                                                 & 0.208$\pm$0.034                                                     & 0.251$\pm$0.033                                                  \\

\textbf{GSM}                                                & 0.716$\pm$0.016                                                    & 0.313$\pm$0.008                                                 & 0.719$\pm$0.014                                                     & 0.359$\pm$0.021                                                  & 0.709$\pm$0.013                                                    & 0.366$\pm$0.008                                                 & \textbf{0.829$\pm$0.019}                                                     & \textbf{0.470$\pm$0.017}                                                  & 0.272$\pm$0.016                                                    & 0.333$\pm$0.013                                                 & 0.304$\pm$0.028                                                     & 0.358$\pm$0.018                                                  \\

\textbf{NSTM}                                             & 0.728$\pm$0.012                                                    & 0.288$\pm$0.007                                                 & 0.755$\pm$0.026                                                     & 0.304$\pm$0.020                                                  & 0.518$\pm$0.013                                                    & 0.221$\pm$0.013                                                 & 0.670$\pm$0.019                                                     & 0.292$\pm$0.009                                                  & 0.272$\pm$0.010                                                    & 0.322$\pm$0.013                                                 & 0.340$\pm$0.021                                                     & 0.375$\pm$0.016                                                  \\
\textbf{vNVDM}                                             & \underline{0.814$\pm$0.009   }                                                 & \underline{0.372$\pm$0.012    }                                             & \underline{0.810$\pm$0.011}                                                     & \underline{0.397$\pm$0.016 }                                                 & \underline{0.793$\pm$0.014 }                                                   & \underline{0.368$\pm$0.010   }                                              & 0.788$\pm$0.016                                                     & 0.392$\pm$0.010                                                  & \underline{0.389$\pm$0.020  }                                                  & \underline{0.413$\pm$0.014  }                                               & \underline{0.371$\pm$0.024 }                                                    & \underline{0.425$\pm$0.015      }                                            \\
\textbf{prodLDA}                                            &                                                             &                                                          & 0.562$\pm$0.051                                                     & 0.117$\pm$0.065                                                  &                                                             &                                                          & 0.355$\pm$0.105                                                     & 0.105$\pm$0.105                                                  &                                                             &                                                          & 0.074$\pm$0.015                                                     & 0.038$\pm$0.029                                                  \\

\hline
\textbf{vONT}                                             & \textbf{0.822$\pm$0.025 }                                                   & \textbf{0.404$\pm$0.025 }                                                & \textbf{0.810$\pm$0.030}                                                     & \textbf{0.423$\pm$0.036 }                                                 & \textbf{0.819$\pm$0.016   }                                                 & \textbf{0.411$\pm$0.012  }                                               & \underline{0.820$\pm$0.015    }                                                 & \underline{0.433$\pm$0.014   }                                               & \textbf{0.456$\pm$0.031}                                                    & \textbf{0.504$\pm$0.020  }                                               & \textbf{0.443$\pm$0.033  }                                                   & \textbf{0.519$\pm$0.018   }           \\                                   
\hline
\end{tabular}}
\end{table*}

\tabcolsep=0.09cm

\begin{table*}
\centering

\caption{\label{tab:people1}Classification performance and Topic Diversity Result for vONTSS. Number of topics equal to 20. Figure~\ref{fig:people1} provides box plots for the metrics. CatE does not produce topics, so we do not have a diversity score. CarEx has diversity equal to 1 by design.}
\vskip 0.15in
\scalebox{0.63} {
\centering
\begin{tabular}{| c| c c c| c c c| c c c |} 
\hline
\multicolumn{1}{c}{} \vline &
\multicolumn{3}{c}{AgNews}\vline  & \multicolumn{3}{c}{20News} \vline & \multicolumn{3}{c}{DBLP}\vline\\
\multicolumn{1}{c}{Method} \vline &
\multicolumn{1}{c}{accuracy} & \multicolumn{1}{c}{micro F1} & \multicolumn{1}{c}{diversity} \vline
& \multicolumn{1}{c}{accuracy} & \multicolumn{1}{c}{micro F1} & \multicolumn{1}{c}{diversity} \vline
& \multicolumn{1}{c}{accuracy} & \multicolumn{1}{c}{micro F1} & \multicolumn{1}{c}{diversity} \vline\\
\hline
\multicolumn{1}{c}{gONTSS} \vline & 0.793 $\pm$ 0.017 & 0.788 $\pm$ 0.017 & \textbf{0.79 $\pm$ 0.025}  & 0.385 $\pm$ 0.129 & 0.327 $\pm$ 0.121 & \underline{0.859 $\pm$ 0.029} & 0.509 $\pm$ 0.040 & 0.457 $\pm$ 0.077 & \underline{0.804 $\pm$ 0.043} \\

\multicolumn{1}{c}{vONTSS with CE} \vline  & 0.754 $\pm$ 0.009 & 0.717 $\pm$ 0.115 & 0.683 $\pm$ 0.009 &  0.468 $\pm$ 0.137 & 0.417 $\pm$ 0.13  & 0.764 $\pm$ 0.046 &  0.547 $\pm$ 0.009 & 0.493 $\pm$ 0.076 & 0.521 $\pm$ 0.064 \\

\multicolumn{1}{c}{vONTSS with all loss} \vline & 0.741 $\pm$ 0.072 & 0.702 $\pm$ 0.125 & 0.652 $\pm$ 0.051 & 0.473 $\pm$ 0.095 & 0.416 $\pm$ 0.112 & 0.729 $\pm$ 0.050 & \underline{0.590 $\pm$ 0.014} & \underline{0.541 $\pm$ 0.048}  & 0.716 $\pm$ 0.12 \\

\multicolumn{1}{c}{CarEx} \vline & 0.778 $\pm$ 0.003 & 0.778 $\pm$ 0.003 &     & 0.44 $\pm$ 0.039  & 0.443 $\pm$ 0.062 &        & 0.530 $\pm$ 0.009 & 0.491 $\pm$ 0.010 &        \\

\multicolumn{1}{c}{CatE} \vline & \underline{0.820 $\pm$ 0.001} & \textbf{0.822 $\pm$ 0.001} &  & \textbf{0.596 $\pm$ 0.002} & \textbf{0.621 $\pm$ 0.002} & &  0.518 $\pm$ 0.001 & 0.536 $\pm$ 0.001 &               \\

\multicolumn{1}{c}{GuidedLDA} \vline & 0.733 $\pm$ 0.037 & 0.735 $\pm $0.039 & 0.561 $\pm$ 0.036 & 0.554 $\pm$ 0.024 & 0.474 $\pm$ 0.026 & 0.584 $\pm$ 0.021 & 0.493 $\pm$ 0.009 & 0.47 $\pm$ 0.008  & 0.314 $\pm$ 0.025\\

\multicolumn{1}{c}{Best Unsupervised} \vline & 0.799 $\pm$ 0.014 & 0.797 $\pm$ 0.015 & 0.573 $\pm $0.049  & 0.501 $\pm$ 0.047  & 0.429 $\pm$ 0.042 & \textbf{0.952 $\pm$ 0.026} & 0.517 $\pm$ 0.037   & 0.377 $\pm$ 0.031 & 0.781 $\pm$ 0.12\\

\multicolumn{1}{c}{Guided BERTopic} \vline & 0.666 $\pm$ 0.023  & 0.573 $\pm $0.049 & 0.487 $\pm$ 0.041   & \underline{0.591 $\pm$ 0.011}  & 0.407 $\pm$ 0.016 & 0.617 $\pm$ 0.031 & 0.486 $\pm$ 0.112   & 0.301 $\pm$ 0.076 & 0.717 $\pm$ 0.07\\

\hline
\multicolumn{1}{c}{vONTSS} \vline  & \textbf{0.823 $\pm$ 0.003} & \underline{0.821 $\pm$ 0.017} & \underline{0.71 $\pm$ 0.024}  & 0.590 $\pm$ 0.014 & \underline{0.554 $\pm$ 0.013} & \underline{0.92 $\pm$ 0.027} & \textbf{0.606 $\pm$ 0.032} & \textbf{0.576 $\pm$ 0.026} & \textbf{0.871 $\pm$ 0.036} \\
\hline

\end{tabular}}

\end{table*}

\begin{table*}[ ]
\centering
\caption{Coherence metrics for vONT. Number of topics is 20. Figure~\ref{fig:people1} in the appendix shows details of the result.}
\vskip 0.15in
\scalebox{0.9} {
\centering
\begin{tabular}{| c| c c| c c| c c } 
\hline
\multicolumn{1}{c}{} \vline &
\multicolumn{2}{c}{AgNews}\vline  & \multicolumn{2}{c}{R8} \vline & \multicolumn{2}{c}{20News}\vline\\
\multicolumn{1}{c}{Method} \vline &
 \multicolumn{1}{c}{$C_{v}$} & \multicolumn{1}{c}{NPMI} \vline
& \multicolumn{1}{c}{$C_{v}$} & \multicolumn{1}{c}{NPMI} \vline
& \multicolumn{1}{c}{$C_{v}$} & \multicolumn{1}{c}{NPMI}\vline \\
\hline
\multicolumn{1}{c}{GSM} \vline & 0.41 $\pm$ 0.01 & \underline{0.03} $\pm$ 0.01 & 0.61 $\pm$ 0.05 & 0.05 $\pm$ 0.01  &  \underline{0.55} $\pm$ 0.04 & \underline{0.07} $\pm$ 0.03\\
\multicolumn{1}{c}{ETM} \vline & 0.41 $\pm$ 0.04 & 0.02 $\pm$ 0.002    & 0.35 $\pm$ 0.02 & -0.04 $\pm$ 0.01  & 0.51 $\pm$ 0.02 & 0.06 $\pm$ 0.01\\
\multicolumn{1}{c}{vNVDM} \vline  & \underline{0.44} $\pm$ 0.02 & 0.028 $\pm$ 0.008   & \textbf{0.74} $\pm$ 0.02 & \underline{0.08} $\pm$ 0.007 & 0.52 $\pm$ 0.01 & 0.03 $\pm$ 0.01\\
\multicolumn{1}{c}{ProdLDA} \vline  & 0.32 $\pm$ 0.04 & -0.22 $\pm$ 0.04   &  0.59 $\pm$ 0.06 & 0.01 $\pm$ 0.003 &  0.35 $\pm$ 0.02 & -0.18 $\pm$ 0.03\\
\multicolumn{1}{c}{NSTM} \vline &  0.37 $\pm$ 0.02  & -0.04 $\pm$ 0.02   & 0.61 $\pm$ 0.01 & -0.08 $\pm$ 0.007 &  0.38 $\pm$ 0.01 & 0.06 $\pm$ 0.04\\
\hline
\multicolumn{1}{c}{vONT} \vline &  \textbf{0.49} $\pm$ 0.02 & \textbf{0.054} $\pm$ 0.02   &  \underline{0.70} $\pm$ 0.03 & \textbf{0.10} $\pm$ 0.03 &  \textbf{0.69} $\pm$ 0.03 & \textbf{0.16} $\pm$ 0.02\\
\hline
\end{tabular}}

\label{table:2}
\end{table*}

\section{Experiment}

\textbf{Dataset} Our experiments are conducted on four widely-used benchmark datasets for topic modeling and semi-supervised text classification with varied length: \textbf{DBLP} \cite{DBLP:conf/ijcai/PanWZZW16}, \textbf{AgNews} \cite{zhang2016characterlevel} and \textbf{20News} \cite{lang1995newsweeder}. All these datasets have ground truth labels. Average document length varies from 5.4 to 155. We preprocess all the datasets by cleaning and tokenizing texts. We remove stop words, words that appear more than 15 percent of all documents and words that appear less than 20 time. For semi-supervised experiments, we use the same labels in DBLP and AgNews. We sample 4 similar classes from 20News to see how our method performs in datasets with similar labels. For unsupervised settings, we keep the number of topics equal to the number of classes plus one. I keep the unit of the length to 10 for all experiments. For semi-supervised settings, we set the number of topics equal to the number of classes in semi-supervised cases, and we provide 3 keywords for each class. We use 20\% as the training set to get our keywords with the top tfidf score for each class. We use 80\% data as the test set. Additional details and provided keywords on the dataset are available in Appendix~\ref{appendix:dataset}

\textbf{Settings} In our experiment setting, we do not utilize any external information beyond the dataset itself. The embedding is trained on the test set. We do not compare methods that rely on transfer learning or language models such as \cite{Bianchi2021PretrainingIA, Yu2021FineTuningPL,Wang2021XClassTC} because of reasons mentioned in appendix~\ref{language}.
The hyperparameter setting used for all baseline models and vONT is similar to \cite{JMLR:v20:18-569}. We use a fully-connected neural network with two hidden layers of [256, 64] unit and ReLU as the activation function followed by a dropout layer (rate = 0.5). We use Adam~\cite{kingma2017adam} as the optimizer with learning rate 0.002 and use batch size 256. We use \cite{smith2018superconvergence} as scheduler and use learning rate 0.01 for maximally iterations equal to 50. We use spherical embeddings~\cite{meng2019spherical} trained on the dataset for NVTM, ETM, GSM and NSTM. For vONT, we set the radius of vMF distribution equal to 10. We fix $\alpha = \delta = 1$ in $L(X,T)$ . We keep $\lambda = 0.01$ in $L_{OT}$.    
Our code is written in PyTorch and all the models are trained on AWS using ml.p2.8xlarge (NVIDIA K80).\footnote{Details on codebases used for baselines and fine-tuning are provided in Appendix  \ref{appendix:code}}
\begin{figure}
\centering
\includegraphics[scale = 0.33]{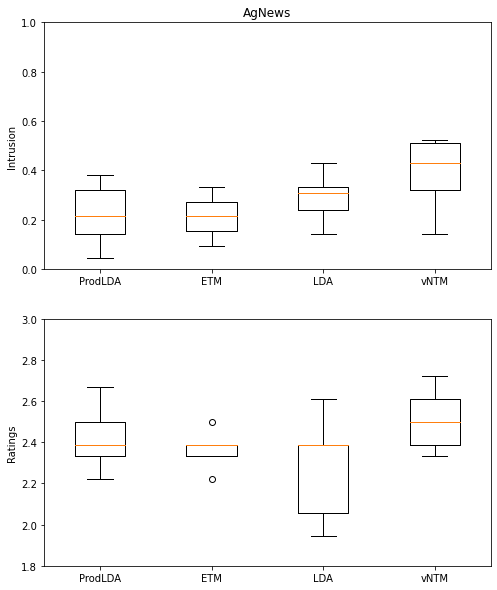}
\caption{Comparison of intrusion and rating task performance on AgNews.}
\label{score}
\end{figure}

\subsection{Unsupervised vONT experiments}
\textbf{Evaluation Metrics}
We measure the topic  coherence and diversity of the model. Most of unsupervised topic coherence metrics are inconsistent with human judgment, based on a recent study~\cite{hoyle2021automated}. Thus, we have done a  qualitative study where we ask crowdsource to perform rating and intrusion task on 4 models trained on AgNews. In rating task\cite{aletras2013evaluating, newman2010automatic, mimno2011optimizing}, raters see a topic and then give the topic a quality score on a three-point scale. The rating score is between 1 and 3. A rating score close to 3 means that users can see a topic from provided words. Chang\cite{chang2009reading} devise the intrusion task, where each topic is represented as its top words plus one intruder word which has a low probability belonging to that topic. Topic coherence is then judged by how well human annotators detect the intruder word. The intrusion score is between 0 and 1. An intrusion score close to 1 means that users can easily identify the intruder word. We use mechanical turk and sagemaker groundtruth to do the labeling work.  To measure clusterability, we assign every document the topic with the highest probability as the clustering label and compute \textbf{Top-Purity} and Normalized Mutual Information(\textbf{Top-NMI}) as metrics\cite{nguyen2018improving} to evaluate alignment. Both of them range from 0 to 1. A higher score reflects better clustering performance. We further apply the KMeans algorithm to topic proportions z and use the clustered documents to report purity(\textbf{Km-Purity}) and NMI \textbf{Km-NMI} \cite{zhao2020neural}. We varied the number of topics from 10 to 50. We set the number of clusters to be the number of topics for KMeans algorithm. Models with higher clusterability are more likely to perform well in semi-supervised extension. Furthermore, we run all these metrics 10 times. We report mean and standard deviation.  Detailed metric implementations are in Appendix \ref{appendix:metric}. We also analyze topic diversity in \ref{diveristy} and unsupervised topic coherence in \ref{coherence}. For Diversity,

\textbf{Baseline Methods}
We compare with the state-of-the-art NTM methods that do not rely on a large neural networks to train. These methods include: \textbf{GSM}~\cite{miao2018discovering}, an NTM replaces the Dirichlet-Multinomial parameterization in LDA with Gaussian Softmax; \textbf{ProdLDA}~\cite{srivastava2017autoencoding}, an NTM model which keeps the Dirichlet Multinomial parameterization with a Laplace approximation; \textbf{ETM}~\cite{dieng2020topic}, an NTM model which incorporates word embedding to model topics; \textbf{vNVDM}~\cite{xu2018spherical}, a vMF based NTM as mentioned in section 2. \textbf{NSTM}~\cite{zhao2020neural}, optimal transport based NTM, as mentioned in section 3. All baselines are implemented carefully with the guidance of their official code.\footnote{Some methods we tested had lower TC scores compared to other benchmarks. This may be because we have less complicated layers, small epochs to train, and we keep fewer words. The ranking of these metrics is mostly in alignment with the paper that has a benchmark. We exclude methods that need to rely on large neural networks and a lot of finetune such as \cite{duan2021sawtooth,duan2021topicnet}. We also exclude methods similar to existing methods such as \cite{Wang2022RepresentingMO}. We exclude methods that do not perform well in previous papers' experiments \cite{duan2021sawtooth} such as \cite{JMLR:v20:18-569}. We also exclude methods that are relevant but work on different use cases, such as short text.\cite{wu-etal-2020-short}} 
For qualitative study, we choose \textbf{ProdLDA}, \textbf{ETM} and \textbf{LDA} as a comparison to align with previous study~\cite{hoyle2021automated}.

\textbf{Results} i) In Table \ref{tab:people20}, vONT performs significantly better than other methods in all datasets for cluster quality metrics. This means vMF distribution induces good clusterability.   ii) vONT has the lowest variance in clusterability-related metrics. (iii) In Appendix \ref{coherence}, vONT outperforms other models in TC metrics $C_v$ and NPMI. This means that our model is coherent. We believe the introduction of the temperature function helps our method perform better than the existed method in coherence. iv) In Appendix \ref{diveristy}, vONT performs well on diversity and has the lowest variance.

\textbf{Human Evaluation} To evaluate human interpretability, we use intrusion test and ratings test. Details of the experiment are provided in Appendix~\ref{humanEval:sec}. We select AgNews as our dataset, we generate 10 topics each from 4 models. In the word intrusion task, we sample five of the ten topic words plus one intruder randomly sampled from the dataset; for the rating task, we present the top ten words in order. Figure~\ref{score} summarizes the results.

vONT performs significantly better than ProdLDA, ETM, and LDA qualitatively. In \textit{intrusion test}, vONT has the highest score 0.4. The second-best method is LDA, which has score 0.29. The two sample test between the two methods has the p-value equal to 0.014.  In \textit{rating test}, vONT has the highest score 2.51 while ProdLDA has the second-highest score 2.42.  The two sample test between the two methods has a p-value equal to 0.036. \textit{Based on this study, we conclude that humans find it easier to interpret topics produced by vONT.}

\subsection{Semi-Supervised vONTSS experiments}
\textbf{Evaluation Metric}
\textbf{diversity} aims to measure how diverse the discovered topic is. \textbf{diversity} is defined as the percentage of unique words in the top 25 words from all topics.\cite{dieng2020topic} \textbf{diversity} close to 0 means redundant and TD close to 1 means varied topics.  We measure the classification accuracy of the model. Thus, we measure \textbf{accuracy}. Similar to other semi-supervised paper\cite{meng2018weakly}, we also measure \textbf{micro f1} score, since this metric gives more information in semi-supervised cases with unbalanced data. We do not include any coherence metric since we already have ground truth.

\textbf{Baseline methods}
\textbf{CatE} \cite{Meng2020DiscriminativeTM}  retrieves category representative terms according to both embedding similarity and distributional specificity. It uses WeSTClass\cite{Meng_2018} for all other steps in weakly-supervised classification. If we do not consider methods with transfer learning or external knowledge, it achieves the best classification performance. \textbf{GuidedLDA} \cite{jagarlamudi-etal-2012-incorporating}: incorporates keywords by combining the topics as a mixture of a seed topic and associating each group of keywords with a multinomial distribution over the regular topics. Correlation Explanation \textbf{CorEx} \cite{gallagher2018anchored} is an information theoretic approach to learning latent topics over documents by searching for topics that are ”maximally informative” about a set of documents. We fine-tune on the training set and choose the best anchor strength parameters for our reporting. We also created semi-supervised ETM by using gaussian distribution and adding the same optimal transport loss as vONTSS. We call it \textbf{gONTSS}. We also train all objectives instead of using two-stage training and call it \textbf{vONTSS with all loss}. Instead of applying optimal transport, we apply cross entropy directly after stage 1 and match topics by keywords set with the highest similarity. We call this method \textbf{vONTSS with CE}. 
To get \textbf{Best Unsupervised} method, we train the unsupervised models(ETM, vNVDM, vONT, ProdLDA) and consider all potential matching between topics and seed words. We report the method with the highest accuracy for each dataset across all different matching. \textbf{Guided BERTopic} We evaluate the guided version of BERTopic \cite{https://doi.org/10.48550/arxiv.2203.05794} method. They create seeded embeddings to find the most similar document. It then takes seed words and assigns them a multiplier larger than 1 to increase the IDF value.  \footnote{We do not find code for other neural-based semi-supervised topic modeling methods~\cite{gemp2019weakly,wang2021neural,Harandizadeh_2022}, but based on their experiments, the best one is \cite{Harandizadeh_2022} which is almost the same as vONTSS with CE which means it has similar variance and lower performance compare to vONTSS with CE  }

\textbf{Results} Table \ref{tab:people1} shows that i) vONTSS outperforms all other semi-supervised topic modeling methods in classification accuracy and micro F1 score, especially for large datasets with lengthy texts such as AgNews. ii) vONTSS has a lower standard deviation compared to other models. This advantage makes our model more stable and practical in real-world applications. iii) To compare methods with/without optimal transport, methods with optimal transport vONTSS achieve much better accuracy, diversity, and lower variance compared to vONTSS with CE and vONTSS with all loss. This means optimal transport does increase the classification accuracy, stability, and diversity of generated topics. iv) In benchmark datasets, vONTSS is comparable to CatE in quality metrics. As can be seen in Table~\ref{speed:tab} in the appendix, vONTSS is 15 times faster than CatE. v) Unsupervised methods cannot produce comparable results even if we use the best topic seed word matching. This shows that semi-supervised topic modeling methods are necessary. vi) Guided Bertopic does not produce good results. It is also not very stable. In Guided Bertopic, the assigned multiplier is increased across all topics, which makes their probability less representative. vi) If we change vONTSS to gONTSS, 

\section{Conclusions}
In this paper, we propose a new semi-supervised neural topic modeling method vONTSS, which leverages vMF, the temperature function, optimal transport, and VAEs. Its unsupervised version exceeds state-of-the-art in topic coherence through both unsupervised and human evaluations while inducing high clusterability among topics. We show that optimal transport loss is equivalent to cross-entropy loss under the optimal condition and induces one-to-one mapping between keywords sets and topics. vONTSS achieves competitive classification performance, maintains top topic diversity, trains fast, and possesses the least variance among diverse datasets. 

\newpage

\newpage
\nocite{langley00}

\bibliography{main}
\bibliographystyle{acl_natbib}

\newpage
\appendix
\onecolumn
\textbf{\Large Appendix} \vspace*{1em} \\

\section{Additional Experimental Results}
Figure~\ref{fig:people20} shows the variation of cluster purity as the number of topics changes. This expands the information provided in Figure~\ref{tab:people20}.

Figure~\ref{fig:people1} provides box plots for the metrics in Table~\ref{tab:people1}.

\begin{figure*}
\includegraphics[scale = 0.31]{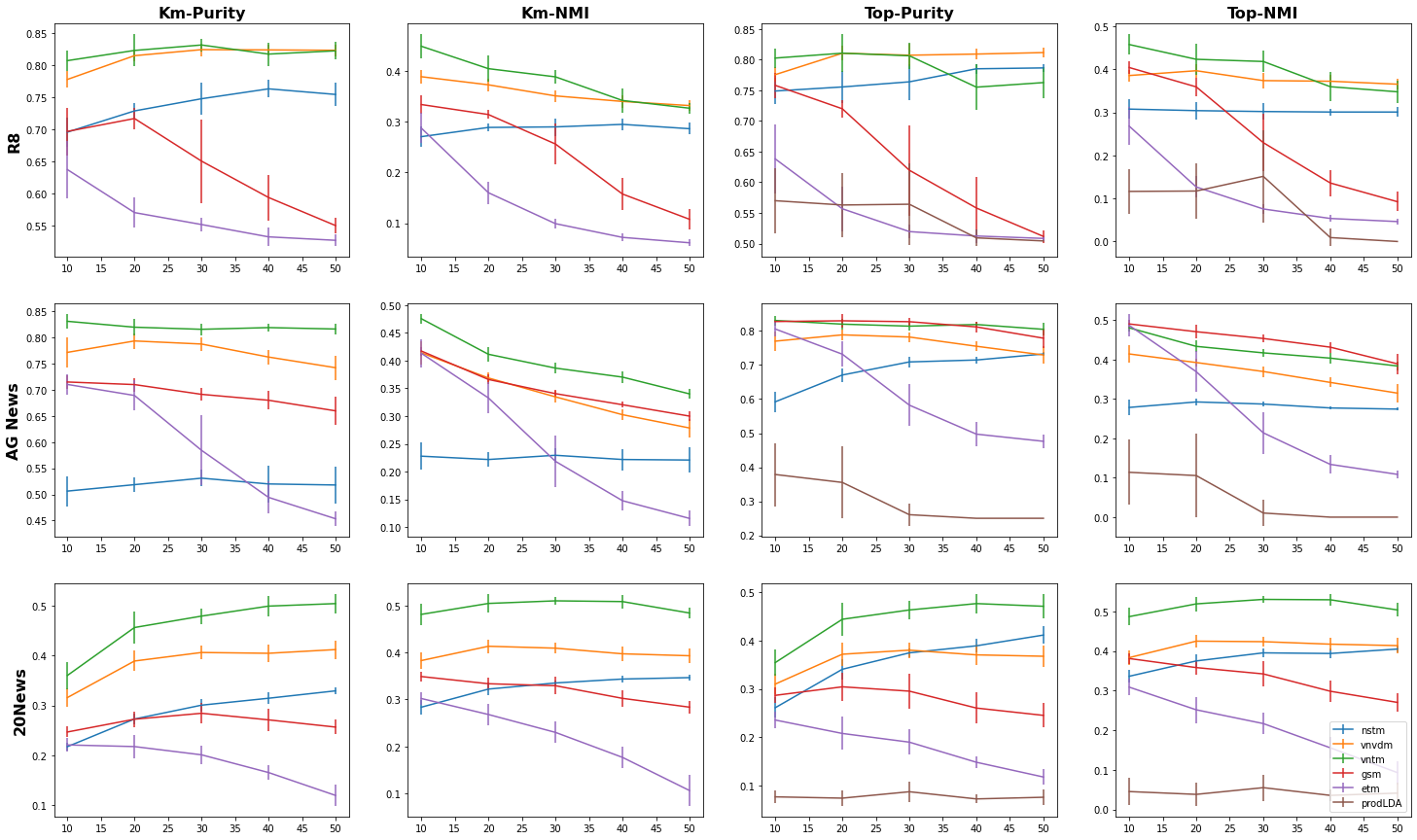}
\centering
\caption{Each column represents a metric and each row represents a dataset. The error bar represents the standard deviation that is created by running the same model for 10 times with different random seeds.}
\label{fig:people20}
\end{figure*}

\begin{figure*}
\includegraphics[scale = 0.38]{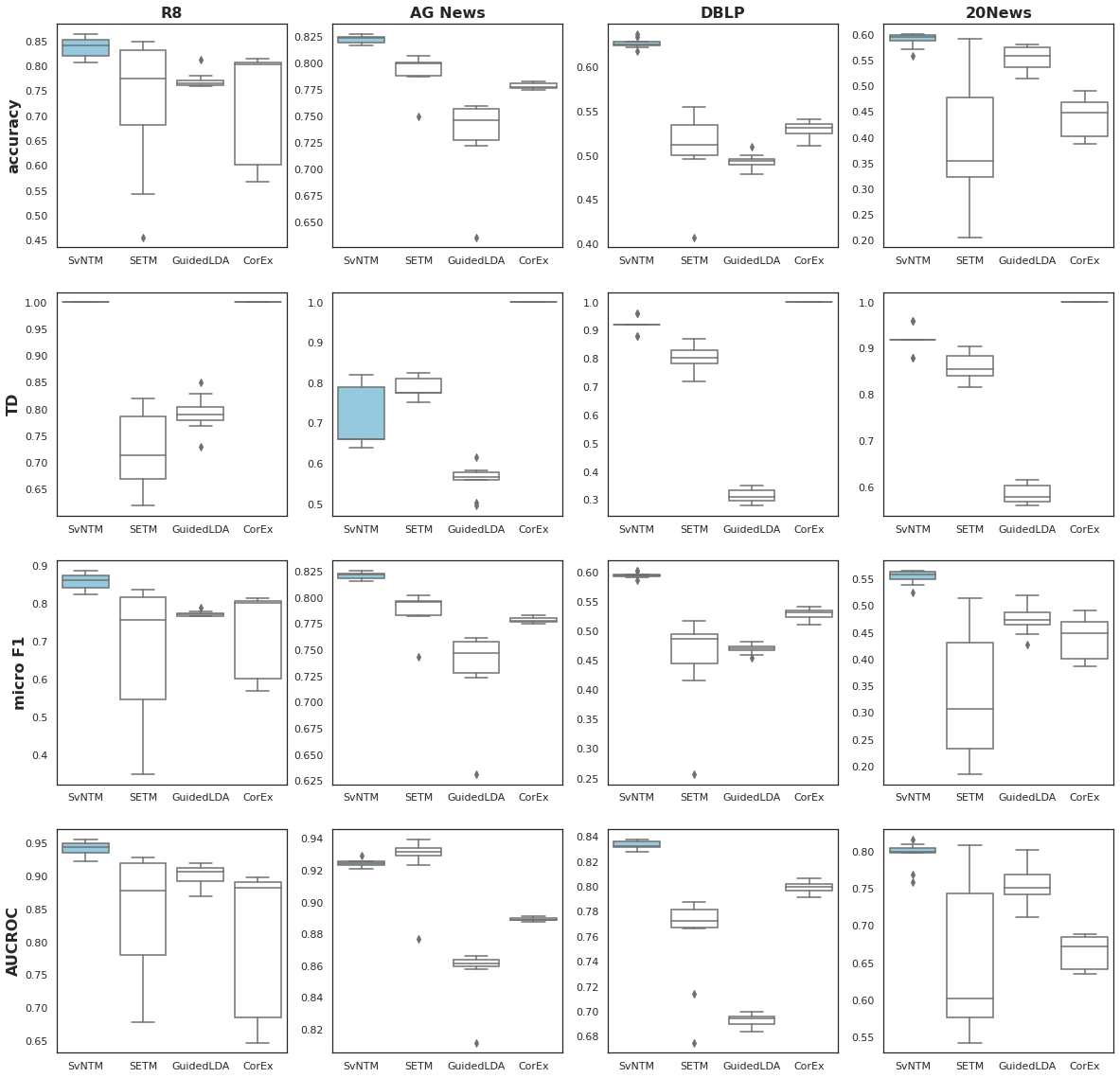}
\caption{Each row represents a metric and each column represent a dataset. The boxplot is created by running the same model for 10 times with different random seeds. Mean and variance values are presented in the boxplot. vONT is the left most. We mark its performance in skyblue. }
\label{fig:people1}
\end{figure*}

\section{Proof of Lemma 3.1}
\label{appendix:proof}
\begin{lemma} When L(X, T) reaches the global minimum. For any $(s, t)$, $(s', t') \in (S, T)$: 
\begin{equation}
\begin{aligned}
    &E_{x \in s} \log q_{\phi}(x|t) +  E_{x \in s'} \log q_{\phi}(x|t')  \\
    &- (E_{x \in s'} \log q_{\phi}(x|t)) + E_{x \in s} \log q_{\phi}(x|t')) >= 0
\end{aligned}
\label{eq21}
\end{equation}

\end{lemma}
\begin{proof}
If the reverse is true, then, we can just switch position of topic t and $t'$ in the topic-word matrix  and also switch the position on latent space z using temperature function. This will not change reconstruction process, since for every input, get the same reconstruction. Thus, reconstruction loss does not change. Assume this new neural network structure has loss $L^{'}(X, T)$ and cross entropy loss is $L^{'}_{ce}$
\begin{dmath}
L^{'}(X, T) - L(X, T)
 = L^{'}_{ce} - L_{ce} 
 = - (E_{x \in s'} \log q_{\phi}(x|t)) + E_{x \in s} \log q_{\phi}(x|t')) 
 + E_{x \in s} \log q_{\phi}(x|t) +  E_{x \in s'} \log q_{\phi}(x|t')   
 < 0  \\
\end{dmath}The last step is based on (9). This contradicts that $L(X, T)$ is global minimal.  Thus, lemma holds.
\end{proof}

\section{Proof of Theorem 3.2}
\begin{theorem}

When L(X, T) reaches the global minimal, $$L_{OT} = L_{ce}  $$
\end{theorem}

\begin{proof}
\textbf{Step 1} show that $p_{t, s} = 1$ when $(t, s) \in (T, S)$ and equal to 0 in all other cases.\\ 
$\exists p_{t, s} = \gamma < 1$ when $(t, s) \in (T, S)$. Without loss of generality, we assume  $p_{t, s'} = 1 - \gamma$, $p_{t', s'} = \gamma$ and $p_{t', s} = 1 - \gamma$. Consider related term in $L_{OT}$, for the first term: 
\begin{align*}
&\gamma (C_{t,s} + C_{t', s'}) + (1 - \gamma) (C_{t,s'} + C_{t', s}) \\
&= (C_{t,s} + C_{t', s'}) \\
&- (1 - \gamma) (C_{t,s} + C_{t', s'} - (C_{t,s'} + C_{t', s})) \\
&\geq C_{t,s} + C_{t', s'} \text{ using Lemma 3.1 and Equation (\ref{eq15})}
\end{align*}
  For the second term in $L_{OT}$, $-p_{t, s} \log p_{t, s}= 0$ when $p_{t, s} = 1$ or 0. Otherwise, it is larger than 0. This means that $p_{t, s} = p_{t', s'} = 1$ achieve smaller $L_{OT}$ compare to current settings. This contradicts the definition of $L_{OT}$ which is the min in the space.  Thus, $p_{t, s} = 1$ when $(t, s) \in (T, S)$. Since the raw sum and column sum equal to $|T|$. This means $p_{t, s} = 0$ when $(t, s) \notin (T, S)$

\textbf{Step 2}: 
\begin{dmath}h(P) = - \sum_{t,s} P_{t,s} \log P_{t,s} \\
 = - (\sum_{(t, s) \in (T,S)} 1 * \log 1 + \sum_{(t, s) \notin (T,S)} 0 * \log 0)  = 0 \end{dmath}

\begin{dmath}\sum_{t,s} P_{t,s} C_{t, s} = \sum_{(t, s) \in (T, S)} C_{t, s} 
=- \sum_{(t, s) \in (T, S)} E_{x\in s} \log q_{\phi}(x|t(x)) \end{dmath}
Combine (10) and (11), we have \begin{dmath}L_{OT} = \sum_{(t,s)\in(S, T)}  C_{t, s} - h(P)= - \sum_{(t,s) \in (T, S)} E_{x \in s} \log q_{\phi}(x|t_{x})  
= L_{ce} \end{dmath}

\end{proof}

\section{Effect of learn-able distribution temperature}
\label{appendix:temperature}
In this study, we make it a learnable parameter and implement it in two ways. The first way is setting temperature variable as one parameter that can be learned (1-p model). All topics share the same parameter. The second way is setting the temperature variable as a vector with dimension equal to the number of topics (n-p model). This means each topic has its own temperature. The initialization value for both the vectors is 10.

After training, the 1-p model has value 4.99 and n-p model has values [-0.45,4.88,5.91,3.47,4.19] (values are rounded to 2 decimals). The accuracy for 1-p model is 78.9 and n-p model is 80.5. This means that vONTSS cannot further improve with learnable temperature. This means that our loss function is not fully aligned with accuracy metric. This is due to the fact that we optimize reconstruction loss as well as KL divergence during the training procedure. This makes our objective less aligned with cross entropy loss.

\section{Code}
\label{appendix:code}
Code we used to implement GSM is \url{https://github.com/YongfeiYan/Neural-Document-Modeling}
Code we used to implement ETM is \url{https://github.com/adjidieng/ETM}
Code we used to implement vNVDM is \url{https://github.com/jiacheng-xu/vmf_vae_nlp} with kl weight = 1 and default scaling item for auxiliary objective term equal to 0.0001
Code we used to implement NSTM is \url{https://github.com/ethanhezhao/NeuralSinkhornTopicModel}
We use same parameters suggested by paper for optimal transport reclossweight = 0.07 and epsilon = 0.001. 
Code we used to implement ProdLDA is \url{https://github.com/vlukiyanov/pt-avitm}
Code we used to implement GSM is \url{https://github.com/YongfeiYan/Neural-Document-Modeling} with  topic covariance penalty equals to 1.
Code we used to implement GuidedLDA is \url{https://github.com/vi3k6i5/GuidedLDA} We fine tune best seed confidence from 0 to 1 with step equal to 0.05. We simply report the best performance on average of 10 results.
Code we used to implement CorEx is \url{https://github.com/gregversteeg/corex_topic}
CorEx are fine-tuned by anchor strength from 1 to 7 with step equal to 1. We simply report the best performance on average of 10 results.
Code we used to implement Spherical Embeddings is \url{https://github.com/yumeng5/Spherical-Text-Embedding}. We set word dimension equals 100, window size equals 10, minimum word count equals 20 and number of threads to be run in parallel equals to 20.The pretrained embedding of all datasets is at the attached data file.
Code we used to implement LDA is \url{https://scikit-learn.org/stable/modules/generated/sklearn.discriminant_analysis.LinearDiscriminantAnalysis.html} with solver = SVD and tol = 0.00001 

Code we used to implement CatE is \url{https://github.com/yumeng5/WeSTClass} and \url{https://github.com/yumeng5/CatE} with number of terms per topic = 10 and text embeddings dimension = 50.  

\section{Coherence}

\label{coherence}

Topic coherence \textbf{TC} metric \cite{mimno2011optimizing} is used to check if topic will include words that tend to co-occur in the same documents. TC \cite{lau2014machine} is the average point wise mutual information (NPMI) of two words drawn randomly from the same documents.  We use both NPMI and $C_{v}$\cite{10.1145/2684822.2685324} by using top 10 words from each topic as suggested in \cite{10.1145/2684822.2685324}.

\section{Diversity Metric}
\label{appendix:metric}
\textbf{diversity} is implemented using scripts: \url{https://github.com/adjidieng/ETM/blob/master/utils.py} line 4. \textbf{$C_{v}$} is implemented using  gensim.models.coherencemodel where coherence = '$C_v$', \textbf{NPMI} is implemented using gensim.models.coherencemodel where coherence = '$c_npmi$'. \textbf{Top-NMI} is implemented using metrics.$normalized_mutual_info_score$ from sklearn. \textbf{Top-Purity} is implemented by definitions. \textbf{km} based is implemented by sklearn package kmeans. 

\section{Datasets}
\label{appendix:dataset}
We store the datasets and related embeddings in the attached data file. 
Overall, we use 4 datasets from different domain to evaluate the performance of our 2 methods.\\ (1) \textbf{AgNews} We use the same AG’s News dataset from \cite{zhang2016characterlevel}.Overall it has 4 classes and, 30000 documents per class. Classes categories include World, Sports, Business, and Sci/Tech. 
for evaluation; 
Keywords we use: 
group1: government,military,war;
group2:basketball,football,athletes;
group3:stocks,markets,industries;
group4:computer,telescope,software
\\
(2) \textbf{R8} is a subset of the Reuters 21578 dataset, which consists of 7674 documents from 8 different reviews groups. We use class acq, earn, and we group all other data in one class. 
Keywords we use:
group1:['acquir', 'acquisit', 'stake'],
group2:['avg', 'mth', 'earn'],
group3:['japan', 'offici', 'export']]

(3) \textbf{20News} \cite{lang1995newsweeder} is a collection of newsgroup posts. We only select 4 categories here. Compare to previous 2 datasets, 4 categories newsgroup is small so that we can check the performance of our methods on small datasets. Keywords we use: group1: faith,accept,world;
group2:evidence,religion,belief;
group3:algorithm,information,problem;
group4:earth,solar,satellite

(4) \textbf{DBLP} \cite{DBLP:conf/ijcai/PanWZZW16} dataset consists of bibliography data in computer science. DBLP selects a list of conferences from 4 research areas, database (SIGMOD, ICDE, VLDB, EDBT, PODS, ICDT, DASFAA, SSDBM, CIKM), data mining (KDD, ICDM, SDM, PKDD, PAKDD), artificial intelligent (IJCAI, AAAI, NIPS, ICML, ECML, ACML, IJCNN, UAI, ECAI,COLT, ACL, KR), and computer vision (CVPR, ICCV, ECCV, ACCV, MM, ICPR, ICIP, ICME). With a total 60,744 papers averaging 5.4 words in each title, DBLP tests the performance on small text corpus. keywords we have: group1: 'system', 'database','query'; group2: 'density', 'nonparametric', 'kernel';
group3: 'image', 'neural', 'recognition'; group4: 'partition', 'group', 'cluster'
\begin{figure*}[h]
\centering
\includegraphics[scale = 0.3]{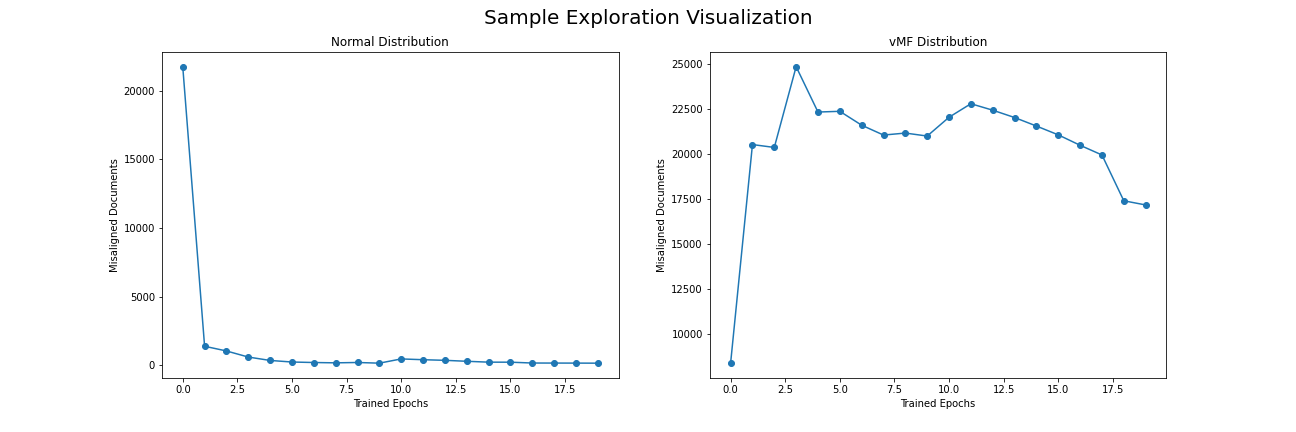}

\caption{Sample Exploration}
\label{fig:22}
\end{figure*}

\begin{figure*}[h]
\centering
\includegraphics[scale = 0.3]{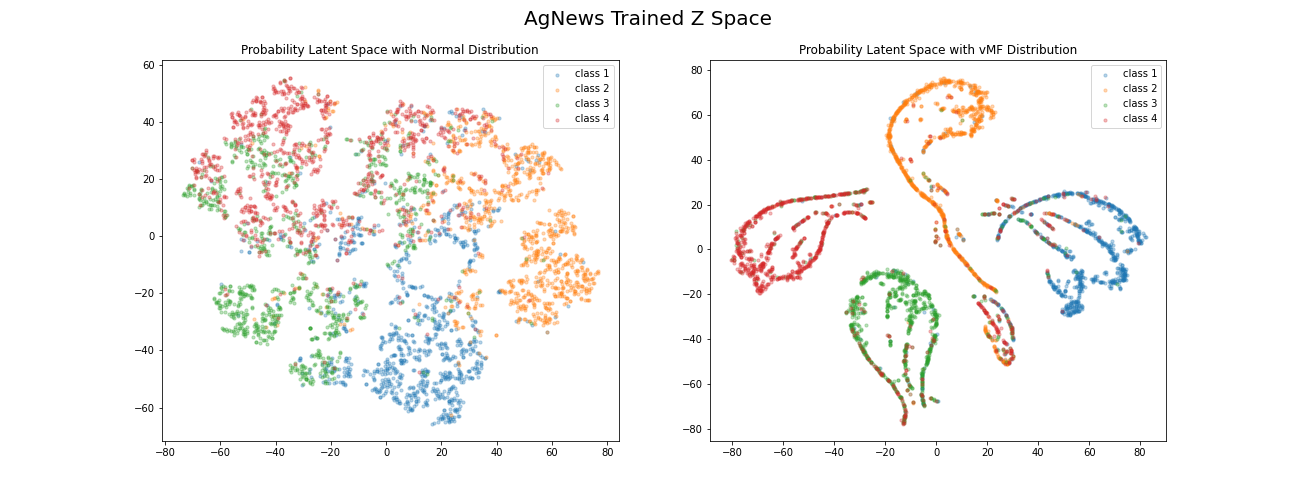}

\caption{Z Space}
\label{fig:people2}
\end{figure*}

\section{Analysis on vMF and Gaussian}
\label{appendix:appendix2}
In this section, we show empirically, vMF encourages topic separation naturally when comparing to Gaussian priors, especially in low dimensions. In the VAE training setting, we have the encoder network $\theta$ learning to transform document inputs $x$ into distribution parameters. Without loss of generality, we denote learned parameters $\vartheta_{i}$ which is updated in the training process and corresponds to latent space $\eta_{i} \sim q(\vartheta_{i})$.

Theoretically, the best $q$ should be able to approximate the posterior distribution $p(\eta_{i}|x)$; however, our choice of parametric distribution family in practice will always associate with our intentions, whether to reduce training time or increase expressability. The choice of prior and posterior distribution can be viewed as a form of regularization on our decoder network, which is arbitrarily powerful. Intuitively, distributions with fewer parameters will introduce more regularization at the cost of less flexibility, analog to bias variance trade off.

For p dimensional latent space, vMF is parameterized by p+1 variables while Gaussian is parameterized by 2*p variables assuming conditional independence or up to p(p+1)/2 + p variables assuming interdependence. In the extreme setting when labelled documents are less than $O(p^2)$, our encoder and decoder may overfit, learning identity mapping.

In the topic modelling space, a softmax transformation $\sigma$ is applied to $\eta$ to extract a probabilistic mixture of topics. In the independent Gaussian posterior case, we view affinity and confidence of the document to topic 1 is encoded in the first entry of $\mu$ and, $\sigma^2$ respectively. Ideally, we would want the encoder to offer variability in the sampling process to regularize, defined as difference in topic probability with initial training epochs; however, we will show through an example \ref{fig:22}, that Gaussian may learn identity mapping by predicting variance to be near 0. 

In the figure below, we define misaligned document as those documents such $argmax(\varsigma) != argmax(\eta)$. This can be viewed as a measure of regularization. In the Gaussian case, our encoder network learns identity mapping within the first epoch. Out of 120000 documents, only 200 or so documents were able to explore different spaces. vMF allows 1/6th of documents to vary and stabilizes after KL divergence kicks in. In trained latent spaces representation, we clearly see vMF learning more nuanced and structured data when comparing to Gaussian as you can see in \ref{fig:people2}


\section{Human Evaluation} \label{humanEval:sec}
We use the ratings and word intrusion tasks as human evaluations of topic quality. We recruit crowdworkers using Amazon Mechanical Turk inside Amazon Sagemaker. We pay workers 0.024 per ratings task and 0.048 per intrusion tasks. We select enough crowdworkers per task  so that p value for two sample t test between the best method and the second-best method is less than 0.05, resulting in a minimal of 18 crowd workers per topic for both tasks. Overall, we ask crowdsources to perform 1641 tasks and create 223 objects. It costs 77.89 for the whole labeling job(Internal price). The user interfaces are shown in Figure~\ref{intrusion} and Figure~\ref{rating}. 
\begin{figure*}
\centering
\includegraphics[scale = 0.3]{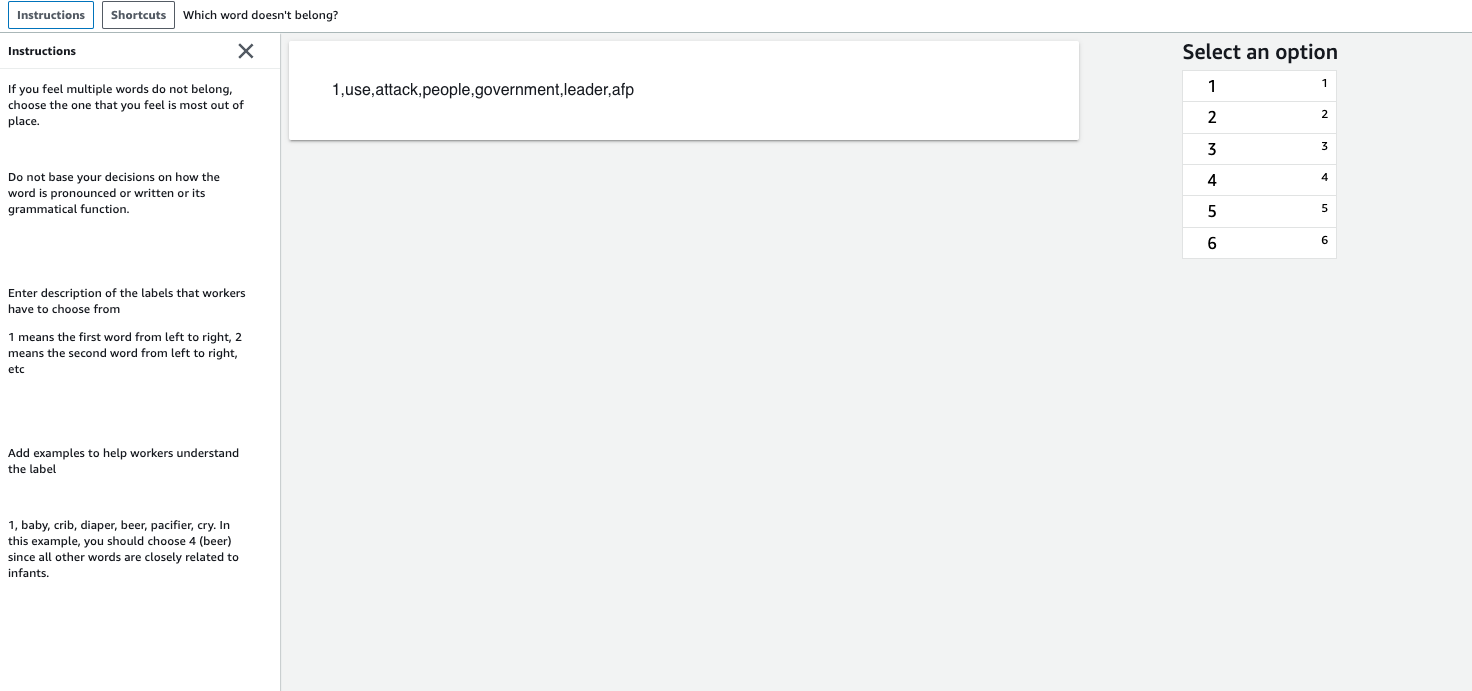}
\caption{User interface of intrusion task }
\label{intrusion}
\end{figure*}

\begin{figure*}
\centering
\includegraphics[scale = 0.3]{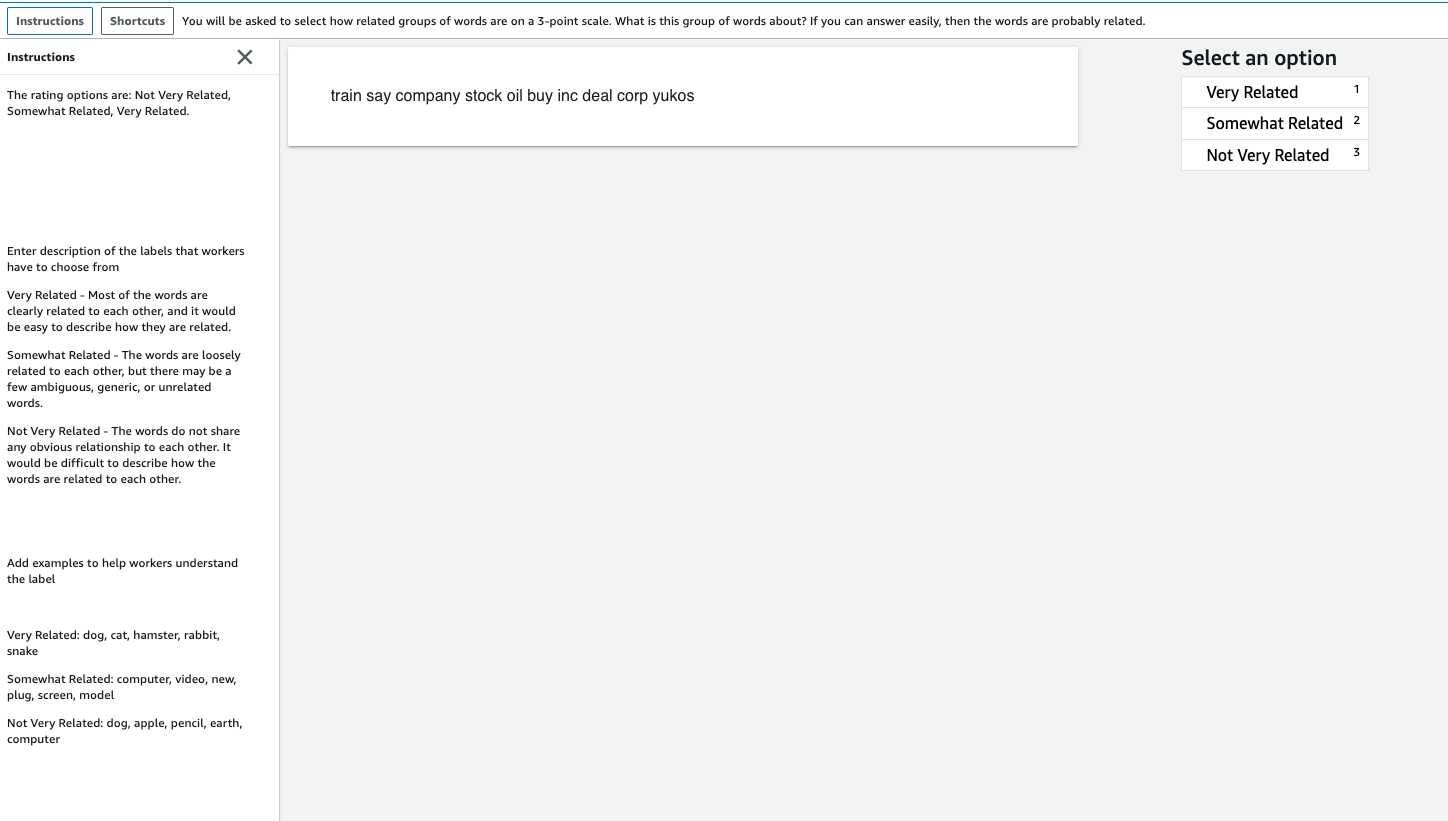}
\caption{User interface of rating task }
\label{rating}
\end{figure*}

We select AgNews as our dataset, we generate 10 topics each from 4 models. In the word intrusion task, we sample five of the ten topic words plus one intruder randomly sampled from the dataset; for the ratings task, we present the top ten words in order.

We also document the confidence per task generated by Amazon Mechanical Turk tool and average time per task for each task as can be seen below. For time spent, crowdsources spend 100 ~ 115 seconds per intrusion task and 70 ~ 80 seconds per rating task. Crowdsources spent 102.7 seconds on intrusion task generated by vONT which is lower than all other tasks. This means that it is easier for users to find intrusion word for topics generated by vONT. The confidence per rating task is between 0.88 to 0.94, where vONT has highest confidence 0.938 while LDA has lowest confidence 0.886. The confidence per intrusion task is between 0.74 to 0.86, where vONT has highest confidence 0.858 while ETM has lowest confidence 0.747. This means the crowdsources are in general more confident in their answer to questions that is generated by vONT.

\begin{figure*}
\centering
\includegraphics[scale = 0.35]{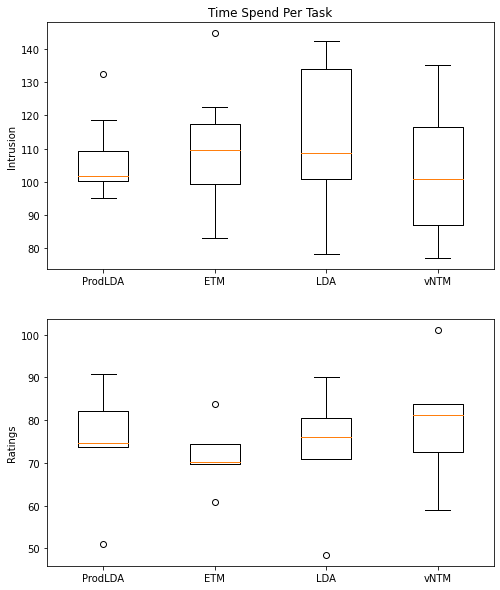}
\caption{Compare different methods' time spend per task }
\label{time}
\end{figure*}

\begin{figure*}
\centering
\includegraphics[scale = 0.35]{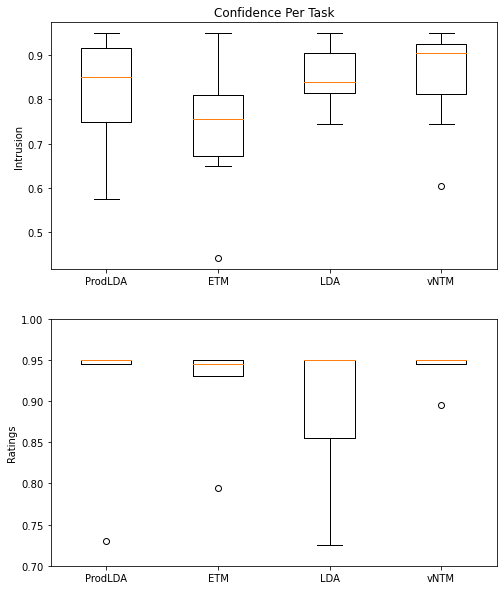}
\caption{Compare different methods'  confidence per task}
\label{confidence}
\end{figure*}

\section{Theoretical Analysis of vMF clusterability}
\label{theory}
In this section, we present theoretical intuition behind cluster inducing property of vMF distribution comparing to the normal distribution.

In the normal VAE set up, the encoder network learns mean parameter $\mu_{i}$ and variance parameter $\sigma_{i}$ for each document i. During the training process, we sample one data point, $\eta_{i}$ from the learned distribution and pass into the softmax function to represent a probability distribution of topics. To introduce high clusterability, we need sampled $\eta$ to have the ability to induce high confidence assignment to a topic under some form of regularization. In other words, with p number of topics, model can increase $argmax(softmax(\eta)) \in (1/p, 1)$ without additional penalty. 

We prove that under normal distribution and in the two dimensional case, it is impossible to increase $argmax(softmax(\eta))$ without increase KL divergence loss with respect to the prior $N(0,I)$. The KL divergence with p = 2 is $KL_{normal} = -\frac{1}{2}[2 + \log\sigma^{2}_{1} + \log\sigma^{2}_{2} - \mu^{2}_{1} - \mu^{2}_{2} - \sigma^{2}_{1} - \sigma^{2}_{2}]$
If we denote $p_{1}$ and $p_{2}$ to be expected distribution of topics, then $p_{1} = \frac{e^{\mu_{1}}}{e^{\mu_{1}} + e^{\mu_{2}}}$ and $p_{2} = \frac{e^{\mu_{2}}}{e^{\mu_{1}} + e^{\mu_{2}}}$. Without loss of generality, we assume that the document i is more aligned with the first topic, the model will learn and output $\mu_{1} > \mu_{2}$. To minimize KL defined above, $\mu_{1}$ and $\mu_{2}$ will be centered be around 0 with $\mu_{1} = -\mu_{2}$; however, in order to increase propensity of $argmax(softmax(\eta))$ or $p_{1}$, $\mu_{1}$ and $\mu_{2}$ have to increase and decrease respectively, forcing the KL divergence penalty to increase.

For vMF distribution, the KL divergence is
$$KL_{vMF} = \kappa\frac{I_{\frac{M}{2}}(\kappa)}{I_{\frac{M}{2}-1}(\kappa)} + (\frac{M}{2} - 1) \log \kappa - \frac{M}{2} \log (2\pi)  - \log I_{\frac{M}{2}-1}(\kappa) $$ $$+ \frac{M}{2} \log \pi + \log 2 + \log \Gamma(\frac{M}{2})$$
We note that the KL penalty under vMF case is not associated with $\mu$, thus the model can increase the propensity without increasing regularization penalties. The KL divergence of vMF
distribution also makes $\kappa$ small, inducing the generated topic
distribution to be localized. If a data point is far different
from any direction parameter $\mu$, the reconstruction loss will
be high as $\kappa$ is small. Thus, $\mu$ should be as representative as
possible which makes it more clustered.

\section{Speed}
We run each model 10 times with different seeds to evaluate how long it takes to finetune the model by modifying 20 percent of keywords set. 
\begin{table*}[ ]

\centering
\begin{tabular}{| c| r| r } 
\hline
\multicolumn{1}{c}{} \vline &
\multicolumn{1}{c}{20News}\vline  & \multicolumn{1}{c}{AgNEWS} \vline \\
\hline
\multicolumn{1}{c}{CorEX} \vline & 2.18 &  94.98\\
\multicolumn{1}{c}{vONTSS} \vline & 0.51 &  51.33 \\
\multicolumn{1}{c}{GuidedLDA} \vline & 1.66 & 24.6 \\
\multicolumn{1}{c}{CatE} \vline & 104.30 & 888.61\\
\hline
\end{tabular}
\caption{\label{speed:tab}Fine Tuning in Seconds}
\end{table*}

\section{Related Works}
\label{vMF}
Most of vMF based topic modeling methods does not incorporate variational autoencoders. Spherical Admixture Model (SAM) \cite{reisinger2010spherical} is the first topic modeling method that uses vMF distribution to model corpus $\mu$, topics and reconstructed documents. Kayhan \cite{batmanghelich2016nonparametric} combines vMF distribution with word embeddings and uses vMF to regenerate the center of topics. It is based on Dirichlet Process to get the proportion of topics for a certain document. Hafsa \cite{9547420} combines knowledge graph and word embeddings for spherical topic modeling. They use vMF distribution to model corpus $\mu$, word embeddings and entity embeddings. To compare, we use modified vMF to generate topic distributions over documents and adapt spherical word embeddings instead of modeling it using vMF. Our method scales well, optimizes fast and offers highly stable performance. The choice of spherical word embeddings also alleviates the sparsity issue among words. vNVDM \cite{xu2018spherical} is the only other method that combines vMF with variational autoencoders. \cite{xu2018spherical} proposes using vMF(.,0) in place of Gaussian as $p(Z)$, avoiding entanglement in the center. They also approximate the posterior $q_{\phi}(Z|X)$ = $vMF(Z;\mu,\kappa)$  where $\kappa$ is fixed to avoid posterior collapse. \textit{The above approach does not work well for two reasons.} Firstly, fixing $\kappa$ causes KL divergence to be constant, which reduces the regularization effect and increases the variance of the encoder.  Another concern with vMF distribution is its limited expressability when its sample is translated into a probability vector. Due to the unit constraint, $softmax$ of any sample of vMF will not result in high probability on any topic even under strong direction $\mu$. For example, when topic dimension $M$ equals to 10, the highest topic proportion of a certain topic is 0.23. We also have a different decoder.

NSTM \cite{zhao2020neural} uses optimal transport to replace KL divergence. Row and column represent topics and words. Instead, our method represents row and column as topics and keywords with M matrix also defined differently. \cite{xu2018distilled} uses optimal transport for topic embeddings, but with wasserstein distances as metric and jointly learns word embeddings. Instead, our algorithm keeps word embedding fixed during the training process to maintain stability.

\section{Ablation Study on Radius}
\label{radius}
Ablation study for radius parameter on AG-News where we set topics equal to 10: as we sweep temperature from 1 to 20, nmi increases and diversity decreases. Radius=10 has the best average rank over coherence based metrics in this temperature range. It has good diversity while has good coherence based metric. Temperature = 10 also has the best pruity score which make it useful for semi-supervised learning 
\begin{table}
\centering

\begin{tabular}{|r|r|r|r|r|r|r|}
\multicolumn{1}{l}{temperature} & \multicolumn{1}{l}{Top-Purity} & \multicolumn{1}{l}{Top-NMI} & \multicolumn{1}{l}{KM-purity} & \multicolumn{1}{l}{KM-NIM} & \multicolumn{1}{l}{NPMI} & \multicolumn{1}{l}{$C_v$}  \\ 
\hline
1                               & 0.70735                    & 0.33626                 & 0.71006                       & 0.3392                     & 0.07476                  & 0.53402                 \\ 
\hline
2                               & 0.7636                     & 0.39615                 & 0.76408                       & 0.39653                    & \textbf{0.10407}                  & 0.60342                 \\ 
\hline
3                               & 0.79176                    & 0.42084                 & 0.79174                       & 0.42093                    & 0.09666                  & 0.61272                 \\ 
\hline
4                               & 0.80763                    & 0.43793                 & 0.80762                       & 0.43774                    & 0.10233                  & 0.62054                 \\ 
\hline
5                               & 0.82157                    & 0.45221                 & 0.82128                       & 0.45177                    & 0.10288                  & \textbf{0.6225}                  \\ 
\hline
6                               & 0.82232                    & 0.45522                 & 0.82221                       & 0.45498                    & 0.09325                  & 0.60377                 \\ 
\hline
7                               & 0.81163                    & 0.45198                 & 0.81151                       & 0.45124                    & 0.08936                  & 0.58558                 \\ 
\hline
8                               & 0.83201                    & 0.46658                 & 0.83202                       & 0.46466                    & 0.09089                  & 0.57832                 \\ 
\hline
9                               & 0.83013                    & 0.47905                 & 0.8301                        & 0.47413                    & 0.07968                  & 0.55148                 \\ 
\hline
10                              & \textbf{0.8353}                     & 0.47956                 & \textbf{0.83524}                       & 0.47599                    & 0.08726                  & 0.5656                  \\ 
\hline
11                              & 0.82824                    & 0.48175                 & 0.82812                       & 0.47662                    & 0.08048                  & 0.55159                 \\ 
\hline
12                              & 0.82555                    & 0.47746                 & 0.82604                       & 0.47259                    & 0.07785                  & 0.54873                 \\ 
\hline
13                              & 0.8268                     & 0.49485                 & 0.82719                       & 0.4852                     & 0.06195                  & 0.50852                 \\ 
\hline
14                              & 0.83481                    & 0.49908                 & 0.83552                       & 0.49035                    & 0.05528                  & 0.50998                 \\ 
\hline
15                              & 0.83189                    & 0.50736                 & 0.83336                       & 0.49562                    & 0.03945                  & 0.48949                 \\ 
\hline
16                              & 0.83049                    & 0.51134                 & 0.83171                       & 0.49854                    & 0.02964                  & 0.48032                 \\ 
\hline
17                              & 0.83023                    & 0.50305                 & 0.83103                       & 0.49246                    & 0.05862                  & 0.50878                 \\ 
\hline
18                              & 0.82232                    & 0.50624                 & 0.8247                        & 0.49461                    & 0.03647                  & 0.48372                 \\ 
\hline
19                              & 0.82955                    & \textbf{0.51175}                 & 0.83167                       & \textbf{0.49915}                    & 0.03496                  & 0.48405                 \\
\hline

\end{tabular}
\caption{\label{temperature} Evaluate the influence of radius on coherence and clusterability related metric in Dataset AgNews. Temperature is from 1 to 20.  The best scores of is highlighted in boldface. The number of topis is 10}

\end{table}

\section{Ablation Study on $\kappa$}
\label{kappas}

Ablation study for Kappa on AG-News: we check kappa = 10, 50, 100, 500, 1000. Kappa=100 has highest purity and nmi, kappa = 50 has highest NPMI and $C_v$. Kappa = 500 has highest diversity. Our version of kappa has highest diversity, purity and NPMI compare to all fixed kappa. 
\begin{table}
\centering
\begin{tabular}{|r|r|r|r|r|r|r|r|}
\multicolumn{1}{l}{kappa} & \multicolumn{1}{l}{diversity} & \multicolumn{1}{l}{Top-Purity} & \multicolumn{1}{l}{Top-Nmi} & \multicolumn{1}{l}{Km-Purity} & \multicolumn{1}{l}{Km-Nmi} & \multicolumn{1}{l}{NPMI} & \multicolumn{1}{l}{$C_v$}  \\ 
\hline
10                        & 0.87                          & 0.78211                    & 0.4333                  & 0.78428                       & \textbf{0.42829}                    & 0.04752                  & 0.51337                 \\ 
\hline
50                        & 0.9904                        & 0.77143                    & 0.41847                 & 0.77266                       & 0.4183                     & 0.05275                  & \textbf{0.52367}                 \\ 
\hline
100                       & 0.9896                        & 0.7832                     & 0.42627                 & 0.78528                       & 0.42654                    & 0.05031                  & 0.51347                 \\ 
\hline
500                       & \textbf{0.9912}                        & 0.78176                    & 0.42274                 & 0.78325                       & 0.42291                    & 0.05113                  & 0.51655                 \\ 
\hline
1000                      & 0.9896                        & 0.76249                    & 0.40888                 & 0.76466                       & 0.40932                    & 0.04678                  & 0.5096                  \\ 
\hline
varied                    & 0.9902                        & \textbf{0.81}                       & \textbf{0.423}                   & \textbf{0.822}                         & 0.404                      & \textbf{0.054}                    & 0.49                    \\
\hline
\end{tabular}

\caption{\label{kappa} Evaluate the influence of learnable on coherence and clusterability related metric in Dataset AgNews. The best scores is highlighted in boldface. The number of topic is 20.}

\end{table}

\section{Diversity Evaluation on vONT}
\label{diveristy}
vONTSS has high diversity by design. As you can see in the table, vONT achieves the best diversity on R8 and AgNews. vONT is the second best on 20News dataset. It also has the lowest standard deviation compare to other methods.

\begin{table}
\centering

\begin{tabular}{|l|l|l|l|} 
\hline
method           & dataset & diversity & std      \\ 
\hline
NSTM           & R8      & 0.3672    & 0.02692  \\ 
\hline
NSTM           & 20News  & 0.55636   & 0.04306  \\ 
\hline
NSTM           & AgNews  & 0.974     & 0.00806  \\ 
\hline
vONT    & R8      & 0.9224    & 0.01613  \\ 
\hline
vONT    & 20News  & 0.9592    & 0.0257   \\ 
\hline
vONT    & AgNews  & 0.99022   & 0.01185  \\ 
\hline
vNVDM   & R8      & 0.52875   & 0.0868   \\ 
\hline
vNVDM  & 20News  & 0.9044    & 0.06152  \\ 
\hline
vNVDM  & AgNews  & 0.6224    & 0.03772  \\ 
\hline
GSM     & R8      & 0.3868    & 0.05126  \\ 
\hline
GSM     & 20News  & 0.6648    & 0.04766  \\ 
\hline
GSM    & AgNews  & 0.576     & 0.02091  \\ 
\hline
ETM      & R8      & 0.1224    & 0.01961  \\ 
\hline
ETM     & 20News  & 0.5024    & 0.03267  \\ 
\hline
ETM       & AgNews  & 0.4896    & 0.04975  \\ 
\hline
ProdLDA & R8      & 0.87429   & 0.05746  \\ 
\hline
ProdLDA & 20News  & 0.97143   & 0.04467  \\ 
\hline
ProdLDA & AgNews  & 0.88286   & 0.08379  \\
\hline
\end{tabular}
\caption{ Evaluate the diversity of vONT compare to other methods in all 3 datasets. The number of topic is 20.}

\end{table}

\section{Why not use language modeling based methods?}
\label{language}
Most language modeling methods are time-consuming to train and need a lot of transfer learning. They also need finetune in most of our use cases.  Without fine-tuning, \cite{Bianchi2021PretrainingIA} makes it harder to be used in domain-specific datasets. We have tried \cite{Yu2021FineTuningPL,Wang2021XClassTC} to compare, but both takes too much time to run. On AG-News, \cite{Yu2021FineTuningPL} takes 108 minutes to run, while \cite{Wang2021XClassTC} takes more than 2.5 hours. It also occurs in other models in footnote 2. vONTSS takes 8 minutes to run and 50 seconds to fine-tune.

We also tried some methods which only leverage embeddings of language modeling such as 
On AgNews and we set topics equal to 20, For \cite{wang-etal-2020-neural-topic}, diversity 0.71, $C_v$ 0.396, NPMI:-0.1089. For \cite{Bianchi2021PretrainingIA}, diversity 1, $C_v$ 0.435, NPMI:-0.1073. Except diversity in \cite{Bianchi2021PretrainingIA}, all other metric perform worse than vONT. 

For semi-spervised cases, we take keywords as input. It is really different from other weakly supervised learning formulations, and how to incorporate keywords into a language model is not straight forward. We have tried few methods, but it does take a lot of time to run and change their code is not easy since their effectiveness do rely on the specific version of language model. Thus, we exclude language modeling methods in our paper. Also, in our use case, each topic model is designed for a specific user or use case. It will be very hard to be interactive or store the model on user's side when the number of parameters is too large for every single model.

 \section{Limitations and Risks}
vMF distribution has a unit constraint. This limits the variability of latent space, which in turn reduces the gains as the number of topics increase. We can try other distributions with richer variability, such as Bivariate von Mises distribution and Kent distribution.

Also, in weakly supervised cases, vONTSS may not perform as well as those methods that leverage pretraining language models in classification. In the future, we can combine the structure of this model with existed language modeling to further improve its classification performance.

Lastly, in semi-supervised cases version, our formulation of vONTSS requires each topic to have at least one keyword. This limits its practical usage to some extent. To solve it, we can first preselect topics before doing the topics and keywords mapping, or we can modify the optimal transport loss using Gumbel distributions. 
\end{document}